\newcount\Comments
\documentclass{article} 

\usepackage{bbm}
\usepackage{amsmath}
\usepackage{amssymb}
\usepackage{algorithm}
\usepackage{mathtools}
\usepackage{amsthm}
\usepackage{natbib}

\PassOptionsToPackage{algo2e,ruled, linesnumbered}{algorithm2e}
\RequirePackage{algorithm2e}

\usepackage[letterpaper,top=2cm,bottom=2cm,left=3cm,right=3cm,marginparwidth=1.75cm]{geometry}

\usepackage[colorlinks=true, allcolors=blue]{hyperref}

\newcommand{\reals}{\mathbb{R}}
\newcommand{\naturals}{\mathbb{N}}

\usepackage{color}
\definecolor{darkgreen}{rgb}{0,0.5,0}
\definecolor{purple}{rgb}{1,0,1}
\newcommand{\kibitz}[2]{\ifnum\Comments=1\textcolor{#1}{#2}\fi}

\newcommand{\Acal}{\mathcal{A}}

\newcommand{\Dcal}{\mathcal{D}}
\newcommand{\Fcal}{\mathcal{F}}

\newcommand{\Hcal}{\mathcal{H}}

\newcommand{\Xcal}{\mathcal{X}}
\newcommand{\Ycal}{\mathcal{Y}}

\newcommand{\Scal}{\mathcal{S}}
\newcommand{\Lcal}{\mathcal{L}}
\newcommand{\Ecal}{\mathcal{E}}
\newcommand{\Pcal}{\mathcal{P}}
\newcommand{\Qcal}{\mathcal{Q}}

\DeclareMathOperator*{\argmin}{arg\,min}

\newcommand{\expect}{\operatorname{\mathbb{E}}}

\newcommand{\indicator}{\mathbbm{1}}

\newcommand{\norm}[1]{\left\lVert#1\right\rVert}

\newtheorem{theorem}{Theorem}[section]
\newtheorem{proposition}[theorem]{Proposition}
\newtheorem{definition}{Definition}
\newtheorem{lemma}[theorem]{Lemma}

\newtheorem*{theorem*}{Theorem}

\title{On the Learnability of Multilabel Ranking}

\author{Vinod Raman, Unique Subedi, Ambuj Tewari}
\date{}

\begin{document}
\maketitle

\begin{abstract}
  Multilabel ranking is a central task in machine learning. However, the most fundamental question of learnability in a multilabel ranking setting with relevance-score feedback remains unanswered. In this work, we characterize the learnability of multilabel ranking problems in both batch and online settings for a large family of ranking losses.
  Along the way, we give two equivalence classes of ranking losses based on learnability that capture most, if not all, losses used in practice.
  

\end{abstract}

\section{Introduction}

\textit{Multilabel ranking} is a supervised learning problem where a learner is presented with an instance $x \in \mathcal{X}$ and is required to output a ranking of $K$ different labels in decreasing order of relevance to $x$. This is in contrast with \textit{multilabel classification} where given an instance $x \in \mathcal{X}$, the learner is tasked with predicting a subset of the $K$ labels without any explicit ordering.  Multilabel ranking is a canonical learning problem with a wide range of applications to text categorization, genetics, medical imaging, social networks, and visual object recognition \citep{joachims2005text, schapire2000boostexter, mccallum1999multi, clare2001knowledge, baltruschat2019comparison, wang2013multi, bucak2009efficient, yang2016exploit}. Recent years have seen a surge in the development of multilabel ranking methods with strong practical and theoretical guarantees \citep{schapire2000boostexter, dembczynski2012consistent, gong2013deep, bucak2009efficient, jung2018online, gao2011consistency, koyejo2015consistent, zhang2013review, korba2018structured}. Despite this vast literature on multilabel ranking, the fundamental question of when a multilabel ranking problem is \textit{learnable} remains unanswered.

Understanding when a hypothesis class is learnable is a fundamental question in Statistical Learning Theory.  For binary classification, the finiteness of the Vapnik–Chervonenkis (VC) dimension is both sufficient and necessary for Probably Approximately Correct (PAC) learning \citep{vapnik74theory, Valiant1984ATO}. Likewise, the finiteness of the Daniely-Shwartz (DS) dimension characterizes multiclass PAC learnability \cite{daniely2014optimal, Brukhimetal2022}.  In the online setting, the Littlestone dimension \citep{Littlestone1987LearningQW} characterizes the online learnability of a binary hypothesis class and the multiclass Littlestone dimension \citep{DanielyERMprinciple} characterizes online multiclass learnability. Unlike classification, a distinguishing property of multilabel ranking is the mismatch between the predictions the learner makes and the feedback it receives. In particular, a learner is required to produce a permutation that ranks the relevance of the labels but only receives a \textit{relevance-score vector} as feedback. This feedback model is standard in multilabel ranking since obtaining full permutation feedback is generally costly \citep{liu2009learning}. As a result, unlike the 0-1 loss in classification, there is no canonical loss function in ranking. Together, these two issues create barriers for existing techniques used to prove learnability, such as the agnostic-to-realizable reductions from \cite{hopkins22a} and \cite{raman2023characterization}, to readily extend to ranking.

In this paper, we characterize the batch and online learnability of a ranking hypothesis class $\mathcal{H} \subset \mathcal{S}_K^{\Xcal}$ under relevance-score feedback, where $\mathcal{S}_K$ is the set of all permutations over $[K] = \{1, ..., K\}$. In doing so, we make the following contributions. 

\begin{itemize} 
\itemsep0em 
\item  We show that a ranking hypothesis class $\mathcal{H}$ embeds $K^2$ different \textit{binary} hypothesis classes $\mathcal{H}_i^j$ for $i, j \in [K]$, where hypotheses in $\mathcal{H}_i^j$ answer whether the label $i$ should be ranked in the top $j$. Our main result relates the learnability of $\mathcal{H}$ to the learnability of $\mathcal{H}_i^j$'s. 

\item We define two families of ranking loss functions that capture most if not all ranking losses used in practice. We show that these families are actually \textit{equivalence} classes - the same characterization of batch and online learnability holds for every loss in that family.

\item By relating the learnability of $\mathcal{H}$ to the learnability of \textit{binary} hypothesis classes $\mathcal{H}_i^j$, we show that existing combinatorial dimensions, like the VC and Littlestone dimension, continue to characterize learnability in the multilabel ranking setting. This allows us to prove that linear ranking hypothesis classes are learnable in the batch setting. 
\end{itemize}


A unifying theme throughout the paper is our ability to \textit{constructively} convert a learning algorithm $\mathcal{A}$ for $\mathcal{H}$ into a learning algorithm $\mathcal{A}^j_i$ for $\mathcal{H}_i^j$ for each $i, j \in [K]$ and vice versa. To do so, our proof techniques involve adapting the agnostic-to-realizable reduction for batch and online classification, proposed by \cite{hopkins22a} and \cite{raman2023characterization} respectively, to ranking.  


\section{Preliminaries and Notation}
Let $\mathcal{X}$ denote the instance space, $\mathcal{S}_K$ the set of permutations over labels $[K]:= \{1, ..., K\}$, and $\mathcal{Y} = \{0, 1, ..., B\}^K$ the target space for some $K, B \in \mathbbm{N}$. We highlight that the set of labels $[K]$ is fixed beforehand and does not depend on the instance $x \in \mathcal{X}$. We refer to an element $y \in \mathcal{Y}$ as a \textit{relevance-score vector} that indicates the relevance of each of the $K$ labels.  Throughout the paper, we treat a permutation $\pi \in \mathcal{S}_K$ as a vector in $\{1, ..., K\}^K$  that induces a \textit{ranking} of the $K$ labels in decreasing order of relevance.
Accordingly, for an index $i \in [K]$, we let $\pi_i \in [K]$ denote the \textit{rank} of label $i$. Likewise, given an index $i \in [K]$, we let $y^i$ denote the relevance of label $i$. In addition, it will be useful to define a mapping from $\mathcal{S}_K$ to $\{0, 1\}^K$. In particular, we define $\text{BinRel}(\cdot, \cdot): \mathcal{S}_K \times [K] \rightarrow \{0, 1\}^K$ as an operator that given a permutation (ranking) $\pi \in \mathcal{S}_K$ and threshold $p \in [K]$, outputs a bit string $b \in \{0, 1\}^K$ s.t. $b_i = \mathbbm{1}\{\pi_i \leq p\}.$ \\

\noindent\textbf{Ranking Equivalences}. Our construction of ranking loss families in Section \ref{sec:lossfam} requires different notions of equivalence between permutations (rankings) in $\mathcal{S}_K$. To that end, we say that $\pi = \hat{\pi}$ iff for all $i \in [K]$, $\pi_i = \hat{\pi}_i$. On the other hand, we say $\pi \stackrel{p}{=} \hat{\pi}$ iff $\{i: \pi_i \leq p\} = \{i: \hat{\pi}_i \leq p\}$. That is, two rankings are $p$-equivalent if the \textit{set} of labels they rank in the top-$p$ are equal. Finally, we say $\pi \stackrel{[p]}{=} \hat{\pi}$ iff for all $j \in [p]$, $\{i: \pi_i \leq j\} = \{i: \hat{\pi}_i \leq j\}$. That is, two rankings are $[p]$-equivalent if not only the \textit{set} but also the \textit{order} of labels they rank in the top-$p$ are equal. \\

\noindent\textbf{Ranking Hypothesis}. A ranking hypothesis $h \in \mathcal{H} \subset \mathcal{S}_K^{\mathcal{X}}$ maps instances in $\mathcal{X}$ to a ranking (permutation) in $\mathcal{S}_K$. Given an instance $x \in \mathcal{X}$, one can think of $h(x)$ as $h$'s ranking of the $K$ different labels in decreasing order of relevance. For any ranking hypothesis $h$, we let $h_i: \mathcal{X} \rightarrow [K]$ denote its restriction to the $i$'th coordinate output. Accordingly, for an instance $x \in \mathcal{X}$, $h_i(x)$ gives the rank that $h$ assigns to label $i$. Given a ranking hypothesis class $\mathcal{H} \subset \mathcal{S}_K^{\mathcal{X}}$ and any $i, j \in [K]$, we define its binary threshold-restricted hypothesis class $\mathcal{H}_i^j = \{h_i^j:  h \in \mathcal{H}\}$ where $h_i^j(x) = \mathbbm{1}\{h_i(x) \leq j\}$. We can think of hypotheses in $\mathcal{H}_i^j$ as providing binary responses to queries of the form: ``for instance $x$, should label $i$ ranked in the top $j$?" These threshold-restricted classes are central to our characterization of learnability in both the batch and online learning settings. \\

\noindent\textbf{Batch Learnability.} In the batch setting, we are interested in characterizing the learnability of a ranking hypothesis class $\Hcal$ under a model similar to the classical PAC model \citep{Valiant1984ATO}.  


\begin{definition}[Agnostic Ranking PAC Learnability] \label{def:batch_learn}
A ranking hypothesis class $\Hcal \subset \mathcal{S}_K^{\mathcal{X}}$ is agnostic PAC learnable w.r.t.\ loss $\ell: \mathcal{S}_K \times \Ycal \to \reals_{\geq 0}$, if there exists a function $m:(0,1)^2 \times \naturals \to \naturals$ and a learning algorithm $\Acal: (\Xcal \times \Ycal)^{\star} \to \Scal_K^{\Xcal}$ with the following property:  for every $\epsilon, \delta \in (0, 1)$ and for every distribution $\Dcal $ on $\Xcal \times \Ycal$, running algorithm $\Acal$ on $n \geq m(\epsilon, \delta, K)$ iid samples from $\Dcal$ outputs a predictor $g=\Acal(S)$ such that with probability at least $ 1-\delta$ over $S \sim\Dcal^{n}$,
\[\expect_{\Dcal}[\ell(g(x), y )] \leq \inf_{h \in \Hcal} \expect_{\Dcal}[\ell(h(x),y)] + \epsilon.\]
\end{definition}
\noindent If $\Dcal$ is restricted to the class of distributions such that $\inf_{h \in \Hcal} \expect_{\Dcal}[\ell(h(x),y )] = 0$, then we say we are in the \textit{realizable} setting. Note that unlike in classification, realizability in the multilabel ranking setting is loss dependent. \\

\noindent\textbf{Online Learnability.} In the online setting, an adversary plays a sequential game with the learner over $T$ rounds. In each round $t \in [T]$, an adversary selects a labeled instance $(x_t, y_t) \in \mathcal{X} \times \mathcal{Y}$ and reveals $x_t$ to the learner. The learner makes a (potentially randomized) prediction $\hat{\pi}_t \in \mathcal{S}_K$. Finally, the adversary reveals the true relevance-score vector $y_t$, and the learner suffers the loss $\ell(\hat{\pi}_t, y_t)$, where $\ell$ is some pre-specified ranking loss function. Given a ranking hypothesis class  $\mathcal{H} \subset \mathcal{S}_K^{\mathcal{X}}$, the goal of the learner is to output predictions $\hat{\pi}_t$ such that its cumulative loss is close to the best possible cumulative loss over hypotheses in $\mathcal{H}$. A hypothesis class is online learnable if there exists an algorithm such that for any sequence of labeled examples $(x_1, y_1), ..., (x_T, y_T)$, the difference in cumulative loss between its predictions and the predictions of the best possible function in $\mathcal{H}$ is small.
 
\begin{definition} [Agnostic Online Ranking Learnability]
\label{def:agnOL}
A ranking hypothesis class $\Hcal \subset \Scal_K^{\Xcal}$ is agnostic online learnable w.r.t. loss $\ell$, if there exists an (potentially randomized) algorithm $\mathcal{A}$ such that for any adaptively chosen sequence of labeled examples $(x_t, y_t) \in \mathcal{X} \times \mathcal{Y}$, the algorithm outputs $\mathcal{A}(x_t) \in \mathcal{S}_K$ at every iteration $t \in [T]$ such that 
$$\mathbb{E}\left[\sum_{t=1}^T \ell(\mathcal{A}(x_t), y_t) - \inf_{h \in \mathcal{H}}\sum_{t=1}^T \ell(h(x_t), y_t)\right] \leq R(T, K) $$
where the expectation is taken w.r.t. the randomness of $\mathcal{A}$ and that of the possibly adaptive adversary, and $R(T, K): \mathbb{N}^2 \rightarrow \mathbb{R}^+$ is the additive regret: a non-decreasing, sub-linear function of $T$.
\end{definition}

\noindent If it is further guaranteed that there exists a hypothesis $h^{\star} \in \Hcal$ such that $\sum_{t=1}^T \ell(h^{\star}(x_t), y_t) = 0$, then we say we are in the \textit{realizable} setting. Again, realizability is loss dependent.

\section{Ranking Loss Families} \label{sec:lossfam}
In statistical learning theory, we often characterize learnability with respect to a loss function. Unlike the 0-1 loss in classification, there is no canonical loss function in multilabel ranking. Accordingly, we define two general families of ranking loss functions in this section and later characterize learnability with respect to all losses in these families. In Appendix \ref{app:catrankloss}, we show that many of the ranking metrics used in practice (e.g. Pairwise Rank Loss, Discounted Cumulative Gain, Reciprocal Rank, Average Precision, Precision@p, etc.)  fall into one of these two families.  

On a high-level, we can classify ranking losses into two main groups: (A) those losses that care about both the order and magnitude of the relevance-scores within the top-$p$ ranked labels and (B) those losses that only care about the magnitude of the relevance-scores within the top-$p$ ranked labels.  Our goal will be to define a loss family for both groups A and B. To do so, we start by identifying a canonical ranking loss that lies in each group. For group A, the normalized sum loss@p, 
$$\ell^{@p}_{\text{sum}}(\pi, y) = \sum_{i=1}^K \min(\pi_i, p+1) y^i - Z^p_y$$
captures both the order and magnitude of the relevance-scores only for the top-$p$ ranked labels. Here, $Z^p_y$ is an appropriately chosen normalization factor that only depends on $p$ and $y$ such that $\min_{\pi \in \mathcal{S}_K}\ell^{@p}_{\text{sum}}(\pi, y) = 0$. For Group B, the normalized precision loss@p, 
$$\ell^{@p}_{\text{prec}}(\pi, y) = Z^p_y - \sum_{i=1}^K\mathbbm{1}\{\pi_i \leq p\} y^i$$
cares only about the magnitude of relevance-scores in the top-$p$ ranked labels. Again, $Z^p_y$ is an appropriately chosen normalization constant that only depends on $p$ and $y$ such that the minimum loss is $0$. The form of $\ell^{@p}_{\text{prec}}$ differs from $\ell^{@p}_{\text{sum}}$ because $\sum_{i=1}^K\mathbbm{1}\{\pi_i \leq p\} y^i$ is a gain whereas $\sum_{i=1}^K \min(\pi_i, p+1) y^i $ is a loss.

Next, we build loss families around $\ell_{\text{sum}}^{@p}$ and $\ell_{\text{prec}}^{@p}$. For $\ell_{\text{sum}}^{@p}$, consider the family: 
$$\mathcal{L}(\ell^{@p}_{\text{sum}}) = \{\ell \in \mathbbm{R}^{\mathcal{S}_K \times \mathcal{Y}}: \ell = 0 \text{ iff } \ell_{\text{sum}}^{@p} = 0\} \cap \{\ell \in \mathbbm{R}^{\mathcal{S}_K \times \mathcal{Y}}: \pi \stackrel{[p]}{=} \hat{\pi} \implies \ell(\pi, y) = \ell(\hat{\pi}, y) \}.$$
By definition, $\mathcal{L}(\ell^{@p}_{\text{sum}})$ contains those ranking losses that are (1) zero-matched with $\ell^{@p}_{\text{sum}}$ and (2) remain unchanged for any two predicted rankings (permutations) that are $[p]$-equivalent. The second constraint is needed to ensure that losses in $\mathcal{L}(\ell^{@p}_{\text{sum}})$ only depend on the order and set  of labels that $\pi$ ranks in the top-$p$. Likewise, we can construct a similar loss family around $\ell^{@p}_{\text{prec}}$ as follows: 
$$\mathcal{L}(\ell^{@p}_{\text{prec}}) = \{\ell \in \mathbbm{R}^{\mathcal{S}_K \times \mathcal{Y}}: \ell = 0 \text{ iff } \ell_{\text{prec}}^{@p} = 0\} \cap \{\ell \in \mathbbm{R}^{\mathcal{S}_K \times \mathcal{Y}}: \pi \stackrel{p}{=} \hat{\pi} \implies \ell(\pi, y) = \ell(\hat{\pi}, y) \}.$$
The set $\mathcal{L}(\ell^{@p}_{\text{prec}})$ contains those ranking losses that are (1) zero-matched with $\ell^{@p}_{\text{prec}}$ and (2) remain unchanged for any two predicted rankings (permutations) that are \emph{$p$-equivalent}. The second constraint is needed to ensure that losses in $\mathcal{L}(\ell^{@p}_{\text{prec}})$ only depend on the set  of labels that $\pi$ ranks in the top-$p$. A major contribution of this paper is showing that both $\mathcal{L}(\ell^{@p}_{\text{sum}})$ and $\mathcal{L}(\ell^{@p}_{\text{prec}})$ are actually \textit{equivalence} classes - the same characterization of learnability holds for every loss in that family.

\section{Batch Multilabel Ranking} \label{sec:batchMLR}

In this section, we characterize the agnostic PAC learnability of hypothesis classes $\mathcal{H} \subset \mathcal{S}_K^{\mathcal{X}}$ with respect to both $\mathcal{L}(\ell^{@p}_{\text{sum}})$ and $\mathcal{L}(\ell^{@p}_{\text{prec}})$. Our main results, stated below as two theorems, relate the learnability of $\mathcal{H}$ to the learnability of the threshold-restricted classes $\mathcal{H}_i^j$. 



\begin{theorem} \label{thm:batch_sum@p} 
 A hypothesis class $\mathcal{H} \subset \mathcal{S}_K^{\mathcal{X}}$ is agnostic PAC learnable w.r.t $\ell \in \mathcal{L}(\ell^{@p}_{\text{sum}})$ \emph{iff}  for all $i \in [K]$ and $j \in [p]$,  $\mathcal{H}_i^j$ is agnostic PAC learnable w.r.t the 0-1 loss. 
\end{theorem}

\begin{theorem}  \label{thm:batch_prec@p}
A hypothesis class $\mathcal{H} \subset \mathcal{S}_K^{\mathcal{X}}$ is agnostic PAC learnable w.r.t $\ell \in \mathcal{L}(\ell^{@p}_{\text{prec}})$ \emph{iff}  for all $i \in [K]$, $\mathcal{H}_i^p$ is agnostic PAC learnable w.r.t the 0-1 loss. 
\end{theorem}

Since VC dimension characterizes the learnability of binary hypothesis classes under the 0-1 loss, an important corollary of Theorems \ref{thm:batch_sum@p} and \ref{thm:batch_prec@p} is that finiteness of $\text{VC}(\mathcal{H}_i^j)$'s, for the appropriate $i, j \in [K] \times [p]$, is necessary and sufficient for agnostic ranking PAC learnability. Later on, we use this fact to prove that linear ranking hypothesis classes are agnostic ranking PAC learnable.  

We start with the proof of Theorem \ref{thm:batch_sum@p}, which follows in three steps. First, we show that if for all $(i, j) \in [K] \times [p]$, $\mathcal{H}_i^j$ is agnostic PAC learnable w.r.t 0-1 loss, then Empirical Risk Minimization (ERM) is an agnostic PAC learner for $\mathcal{H}$ w.r.t $\ell_{\text{sum}}^{@p}$. Next, we show that if $\mathcal{H}$ is agnostic PAC learnable w.r.t $\ell_{\text{sum}}^{@p}$, then $\mathcal{H}$ is agnostic PAC learnable w.r.t any loss $\ell \in \mathcal{L}(\ell_{\text{sum}}^{@p})$. Finally, we prove the necessity direction - if $\mathcal{H}$ is agnostic PAC learnable w.r.t an arbitrary $\ell \in \mathcal{L}(\ell_{\text{sum}}^{@p})$, then for all $(i, j) \in [K] \times [p]$, $\mathcal{H}_i^j$ is agnostic PAC learnable w.r.t 0-1 loss. The proof of Theorem \ref{thm:batch_prec@p} follows exactly the same way as Theorem \ref{thm:batch_sum@p} with some minor changes. Thus, we only focus on the proof of Theorem \ref{thm:batch_sum@p} in this section and defer all discussion of Theorem \ref{thm:batch_prec@p} to Appendix \ref{app:batchprec@p}. 

We begin with Lemma \ref{lem:batch_suffsumloss}, which asserts that if $\mathcal{H}_i^j$ is agnostic PAC learnable for all $(i, j) \in [K] \times [p]$, then ERM is an  agnostic PAC learner for $\mathcal{H}$ w.r.t $\ell_{\text{sum}}^{@p}$.

\begin{lemma} \label{lem:batch_suffsumloss}
If for all $i \in [K]$ and $j \in [p]$, $\mathcal{H}_i^j$ is agnostic PAC learnable w.r.t the 0-1 loss, then \emph{ERM} is an agnostic PAC learner for $\mathcal{H} \subset \mathcal{S}_K^{\mathcal{X}}$ w.r.t $\ell_{\text{sum}}^{@p}$.
\end{lemma}

The proof of Lemma \ref{lem:batch_suffsumloss} exploits the nice structure of $\ell_{\text{sum}}^{@p}$ by upperbounding the empirical Rademacher complexity of the loss class $\ell_{\text{sum}}^{@p} \circ \mathcal{H} = \{(x, y) \mapsto \ell_{\text{sum}}^{@p}(h(x), y): h \in \mathcal{H})\}$ and showing that it vanishes as the sample size $n$ becomes large. Then, standard uniform convergence arguments outlined in Proposition \ref{prop:rad} imply that ERM is an agnostic PAC learner for $\mathcal{H}$ w.r.t $\ell_{\text{sum}}^{@p}$. The full proof is in Appendix \ref{app:batchproofs}.  

Since arbitrary losses in $\mathcal{L}(\ell_{\text{sum}}^{@p})$ may not have nice analytical forms, Lemma \ref{lem:batch_sumloss2arbloss} relates the  learnability of an arbitrary loss $\ell \in \mathcal{L}(\ell_{\text{sum}}^{@p})$ to the learnability of $\ell_{\text{sum}}^{@p}$.

\begin{lemma} \label{lem:batch_sumloss2arbloss}
 If $\mathcal{H} \subset \mathcal{S}_K^{\mathcal{X}}$ is agnostic PAC learnable w.r.t $\ell_{\text{sum}}^{@p}$, then $\mathcal{H}$ is agnostic PAC learnable w.r.t any $\ell \in \mathcal{L}(\ell_{\text{sum}}^{@p})$.
\end{lemma}

\begin{proof} (of Lemma \ref{lem:batch_sumloss2arbloss})
Fix $\ell \in \mathcal{L}(\ell_{\text{sum}}^{@p})$. Let $a = \min_{\pi, y}\{ \ell(\pi, y) \, \mid \, \ell(\pi, y) \neq 0\}$ and $b = \max_{\pi, y}\ell(\pi, y)$. We need to show that if $\mathcal{H}$ is agnostic PAC learnable w.r.t $\ell_{\text{sum}}^{@p}$, then $\mathcal{H}$ is agnostic PAC learnable w.r.t $\ell$. We will do so in two steps. First, we will show that if $\mathcal{A}$ is an agnostic PAC learner for $\ell_{\text{sum}}^{@p}$, then $\mathcal{A}$ is also a \textit{realizable} PAC learner for $\ell$. Next, we will show how to convert a realizable PAC learner for $\ell$ into an agnostic PAC learner for $\ell$ in a black-box fashion. The composition of these two pieces yields an agnostic PAC learner for $\mathcal{H}$ w.r.t $\ell$.

\textbf{Realizable PAC learnability of $\mathcal{H}$ w.r.t $\ell$}. If $\mathcal{H}$ is agnostic PAC learnable w.r.t $\ell_{\text{sum}}^{@p}$, then there exists a learning algorithm $\mathcal{A}$ with sample complexity $m(\epsilon, \delta, K)$ s.t. for any distribution $\mathcal{D}$ over $\mathcal{X} \times \mathcal{Y}$, with probability $1-\delta$ over a sample $S \sim \mathcal{D}^{n}$ of size $n \geq m(\epsilon, \delta, K)$,  the output predictor $g = \Acal(S)$ achieves
$\mathbbm{E}_{\mathcal{D}}\left[\ell_{\text{sum}}^{@p}(g(x), y)) \right] \leq \inf_{h \in \mathcal{H}}\mathbbm{E}_{\mathcal{D}}\left[\ell_{\text{sum}}^{@p}(h(x), y)) \right] + \epsilon.$
In the realizable setting, we are further guaranteed that there exists a hypothesis $h^{\star} \in \mathcal{H}$ s.t.  $\mathbbm{E}_{\mathcal{D}}\left[\ell(h^{\star}(x), y)) \right] = 0$. Since $\ell \in \mathcal{L}(\ell_{\text{sum}}^{@p})$, this also implies that  $\mathbbm{E}_{\mathcal{D}}\left[\ell_{\text{sum}}^{@p}(h^{\star}(x), y)) \right] = 0$. Therefore, under realizability and the fact that $\ell \leq  b\:\ell_{\text{sum}}^{@p}$, we have $\mathbbm{E}_{\mathcal{D}}\left[\ell(g(x), y)) \right] \leq  b\epsilon$. This completes the first part of the proof as we have shown that $\mathcal{A}$ is also a realizable PAC learner for $\mathcal{H}$ w.r.t $\ell$ with sample complexity $m(\frac{\epsilon}{b}, \delta, K)$. 

\textbf{Realizable-to-agnostic conversion}. Now, we show how to convert the realizable PAC learner $\mathcal{A}$ for $\ell$ into an agnostic PAC learner for $\ell$ in a black-box fashion. For this step, we will extend the agnostic-to-realizable reduction proposed by \cite{hopkins22a} to the ranking setting by accommodating the mismatch between the range space of $\mathcal{H}$ and the label space $\mathcal{Y}$. In particular, we will show that Algorithm \ref{alg:batch_sum2arbell} below converts a realizable PAC learner for $\ell$ into an agnostic PAC learner for $\ell$. Note that although input $\mathcal{A}$ is a realizable learner, the distribution $\mathcal{D}$ may not be realizable. 

\begin{algorithm}
\caption{Agnostic  PAC learner  for $\Hcal$ w.r.t. $\ell$}
\label{alg:batch_sum2arbell}
\setcounter{AlgoLine}{0}
\KwIn{Realizable PAC  learner $\Acal$ for $\Hcal$, unlabeled and labeled samples $S_U \sim \Dcal_{\Xcal}^n$ and $S_L \sim \Dcal^m$ }

For each $h \in \mathcal{H}_{|S_U}$, construct a dataset
$$S_U^h = \{(x_1, \tilde{y}_1), ..., (x_n,\tilde{y}_n)\} \text{ s.t. } \tilde{y}_i \sim \text{Unif}\{\text{BinRel}(h(x_i), 1), ..., \text{BinRel}(h(x_i), p)\}$$

Run $\Acal$ over all datasets to get $C(S_U) := \left\{\Acal\big(S_U^h \big) 
 \mid  h \in \mathcal{H}_{|S_U}\right\}$

Return $\hat{g} \in C(S_U)$  with the lowest empirical error over $S_L$ w.r.t. $\ell$. 
\end{algorithm}

Let $h^{\star} = \argmin_{h \in \mathcal{H}} \mathbbm{E}_{\mathcal{D}}\left[\ell(h(x), y) \right]$ denote the optimal predictor in $\mathcal{H}$ w.r.t $\mathcal{D}$. Consider the sample $S_U^{h^{\star}}$ and let $g = \mathcal{A}(S_U^{h^{\star}})$. We can think of $g$ as the output of $\mathcal{A}$ run over an i.i.d sample $S$ drawn from $\mathcal{D}^{\star}$,  a joint distribution over $\mathcal{X} \times \mathcal{Y}$ defined procedurally by first sampling $x \sim \mathcal{D}_{\mathcal{X}}$, then independently sampling $j \sim \text{Unif}([p])$, and finally outputting the labeled sample $(x, \text{BinRel}(h^{\star}(x), j))$.
Note that $\mathcal{D}^{\star}$ is indeed a realizable distribution (realized by $h^{\star}$) w.r.t both $\ell$ and $\ell_{\text{sum}}^{@p}$. Recall that $m_{\Acal}(\frac{\epsilon}{b}, \delta, K)$ is the sample complexity of $\mathcal{A}$. Since $\mathcal{A}$ is a realizable learner for $\mathcal{H}$ w.r.t $\ell$, we have that for $n \geq m_{\Acal}(\frac{a\epsilon}{2b^2p}, \delta/2, K)$, with probability at least $1-\frac{\delta}{2}$, $\mathbbm{E}_{\mathcal{D}^{\star}}\left[\ell(g(x), y) \right] \leq \frac{a\epsilon}{2bp}.$

 Next, by Lemma \ref{lem:subaddsum@p}, we have  $\ell(g(x), y) \leq \ell(h^{\star}(x), y) + \frac{bp}{a}\mathbbm{E}_{j \sim \text{Unif}([p])}\left[\ell(g(x), \text{BinRel}(h^{\star}(x), j)) \right]$ pointwise. Taking expectations on both sides of the inequality gives 
\begin{align*}
\mathbbm{E}_{\mathcal{D}}\left[ \ell(g(x), y) \right] &\leq \mathbbm{E}_{\mathcal{D}}\left[\ell(h^{\star}(x), y) \right] + \frac{bp}{a}\mathbbm{E}_{x \sim \mathcal{D}}\left[\mathbbm{E}_{j \sim \text{Unif}([p])}\left[\ell(g(x), \text{BinRel}(h^{\star}(x), j)) \right] \right]\\
&\leq \mathbbm{E}_{\mathcal{D}}\left[\ell(h^{\star}(x), y) \right] + \frac{\epsilon}{2}.
\end{align*}
The last inequality follows from the definition of $\mathcal{D}^{\star}$, namely $\mathbbm{E}_{\mathcal{D}^{\star}}\left[\ell(g(x), y) \right] = \mathbbm{E}_{x \sim \mathcal{D}_{\mathcal{X}}}\mathbbm{E}_{j \sim \text{Unif}([p])}\left[\ell(g(x), \text{BinRel}(h^{\star}(x), j)) \right]$.
This shows that $C(S_U)$ contains a hypothesis $g$ that generalizes well with respect to $\mathcal{D}$. Now we want to show that the predictor $\hat{g}$ returned in step 4 also has good generalization. Crucially, observe that $C(S_U)$ is a finite hypothesis class with cardinality at most  $2^{nK}$. By standard Chernoff and union bounds, with probability at least $1-\delta/2$, the empirical risk of every hypothesis in $C(S_U)$ on a sample of size $\geq  \frac{8}{\epsilon^2} \log{\frac{4 |C(S_U)|}{\delta}}$ is at most $\epsilon/4$ away from its true error. So, if $m = |S_L| \geq  \frac{8}{\epsilon^2} \log{\frac{4 |C(S_U)|}{\delta}}$, then with probability at least $1-\delta/2$, 
\[\frac{1}{|S_L|} \sum_{(x, y) \in S_L} \ell(g(x), y)  \leq \expect_{\Dcal}\left[\ell(g(x), y)\right] + \frac{\epsilon}{4} \leq \mathbbm{E}_{\mathcal{D}}\left[\ell(h^{\star}(x), y) \right] + \frac{3\epsilon}{4}. \]
Since $\hat{g}$ is the ERM on $S_L$ over $C(S)$, its empirical risk can be at most $\mathbbm{E}_{\mathcal{D}}\left[\ell(h^{\star}(x), y) \right] + \frac{3\epsilon}{4}$. Given that the population risk of $\hat{g}$ can be at most $\epsilon/4$ away from its empirical risk, we have that
\[\expect_{\Dcal}[\ell(\hat{g}(x), y)]  \leq \mathbbm{E}_{\mathcal{D}}\left[\ell(h^{\star}(x), y) \right] + \epsilon. \]
Applying union bounds, the entire process succeeds with probability $1- \delta$.  We can upper bound the sample complexity of Algorithm \ref{alg:batch_sum2arbell}, denoted $n(\epsilon, 
\delta, K)$, as 
\begin{equation*}
    \begin{split}
       n(\epsilon, \delta, K) &\leq m_{\mathcal{A}}\left(\frac{a\epsilon}{2b^2p}, \delta/2, K \right) + O \left( \frac{1}{\epsilon^2} \log{\frac{|C(S_U)|}{\delta}}\right)\\
       &\leq  m_{\mathcal{A}}\left(\frac{a \epsilon}{2b^2p}, \delta/2, K \right) + O \left( \frac{ K m_{\mathcal{A}}(\frac{a\epsilon}{2b^2p}, \delta/2, K)\, \,+ \log{\frac{1}{\delta}}}{\epsilon^2} \right),
    \end{split}
\end{equation*}
where we use $|C(S_U)| \leq 2^{Km_{\mathcal{A}}(\frac{a\epsilon}{2b^2p}, \delta/2, K)}$. This shows that Algorithm \ref{alg:batch_sum2arbell} is an agnostic PAC learner for $\mathcal{H}$ w.r.t $\ell$.
\end{proof}

Finally, Lemma \ref{lem:batch_sum@pnec} gives the necessity direction of Theorem \ref{thm:batch_sum@p}. 


\begin{lemma} \label{lem:batch_sum@pnec}
 If a hypothesis class $\mathcal{H} \subset \mathcal{S}_K^{\mathcal{X}}$ is agnostic PAC learnable w.r.t $\ell \in \mathcal{L}(\ell^{@p}_{\text{sum}})$, then $\mathcal{H}_i^j$ is agnostic PAC learnable w.r.t the 0-1 loss for all $(i, j) \in [K] \times [p]$.
\end{lemma}

Like the sufficiency proofs, the proof of Lemma \ref{lem:batch_sum@pnec} is constructive. 
Given an agnostic PAC learner $\mathcal{A}$ for $\mathcal{H}$ w.r.t $\ell$, we construct an agnostic PAC learner for $\mathcal{H}_i^j$ w.r.t 0-1 loss using a slight modification of Algorithm \ref{alg:batch_sum2arbell}. We defer the full proof to Appendix \ref{app:batchproofs} since the analysis is similar to that of Algorithm \ref{alg:batch_sum2arbell}. Together, Lemmas \ref{lem:batch_suffsumloss}, \ref{lem:batch_sumloss2arbloss} and \ref{lem:batch_sum@pnec} imply Theorem \ref{thm:batch_sum@p}.  

We conclude this section by giving a concrete application of our characterization. Consider the class of ranking-hypotheses  $\mathcal{H} = \{x \mapsto \mathrm{argsort}(Wx):W \in \mathbbm{R}^{K \times d}\}$ that compute rankings by sorting scores, in descending order, obtained from a linear function of the input features. Lemma \ref{lem:example}, whose proof is in Appendix \ref{app:scorerankers}, computes the VC dimension of $\mathcal{H}_i^j$ for an arbitrary $i, j \in [K]$.

\begin{lemma} \label{lem:example}
Let $\mathcal{H} = \{x \mapsto \emph{argsort}(Wx):W \in \mathbbm{R}^{K \times d}\}$ be a linear ranking hypothesis class. Then for all $i, j \in [K]$,  $\text{VC}(\mathcal{H}_i^j) = \tilde{O}(Kd)$, where $\tilde{O}$ hides logarithmic factors of $d$ and $K$.
\end{lemma}

Combining Lemma \ref{lem:example} with Theorems \ref{thm:batch_sum@p} and \ref{thm:batch_prec@p} shows that linear ranking hypothesis classes are agnostic ranking PAC learnable w.r.t to all losses in $\mathcal{L}(\ell^{@p}_{\text{sum}}) \cup \mathcal{L}(\ell^{@p}_{\text{prec}}).$  More generally, in Appendix \ref{app:scorerankers} we give a dimension-based sufficient condition under which \textit{generic} score-based ranking hypothesis classes are agnostic ranking PAC learnable.

\section{Online Multilabel Ranking} \label{sec:onlMLR}

We now move to the online setting and characterize the online learnability of hypothesis classes $\mathcal{H} \subset \mathcal{S}_K^{\mathcal{X}}$ with respect to both $\mathcal{L}(\ell^{@p}_{\text{sum}})$ and $\mathcal{L}(\ell^{@p}_{\text{prec}})$. As in the batch setting, our characterization relates the learnability of $\mathcal{H}$ to the learnability of the threshold-restricted classes $\mathcal{H}_i^j$.

\begin{theorem}\label{thm:ol_sum@p}
A hypothesis class $\mathcal{H} \subset \mathcal{S}_K^{\mathcal{X}}$ is agnostic online learnable w.r.t $\ell \in \mathcal{L}(\ell^{@p}_{\text{sum}})$ \emph{iff} for all $i \in [K]$ and $j \in [p]$, $\mathcal{H}_i^j$ is agnostic online learnable w.r.t the 0-1 loss. 
\end{theorem}

\begin{theorem} \label{thm:ol_prec@p}
A hypothesis class $\mathcal{H} \subset\mathcal{S}_K^{\mathcal{X}}$ is agnostic online learnable w.r.t $\ell \in \mathcal{L}(\ell^{@p}_{\text{prec}})$ \emph{iff}  for all $i \in [K]$, $\mathcal{H}_i^p$ is agnostic online learnable w.r.t the 0-1 loss. 
\end{theorem}

Since the Littlestone dimension characterizes the online learnability of binary hypothesis classes under the 0-1 loss, an important corollary of Theorems \ref{thm:ol_sum@p} and \ref{thm:ol_prec@p} is is that finiteness of $\text{Ldim}(\mathcal{H}_i^j)$, for the appropriate $i, j \in [K] \times [p]$, is necessary and sufficient for agnostic online ranking learnability. 

We now begin the proof of Theorem \ref{thm:ol_sum@p}. Since the proof of Theorem \ref{thm:ol_prec@p} follows a similar trajectory, we defer all discussion of Theorem \ref{thm:ol_prec@p} to Appendix \ref{appdx:proof_ol_prec@p}. Unlike Theorem \ref{thm:batch_sum@p} in the batch setting, we prove the sufficiency and necessity directions of Theorem \ref{thm:ol_sum@p} directly. We chose this direct path because, unlike the batch setting, sequential Rademacher analysis does not yield a constructive algorithm \citep{rakhlin2015sequential}.
On the other hand, our proofs are constructive and use the celebrated Randomized Exponential Weights Algorithm (REWA) \citep{cesa2006prediction}. 

\begin{proof} (of sufficiency in Theorem \ref{thm:ol_sum@p})
 Fix $\ell \in \Lcal(\ell_{\text{sum}}^{@p})$. Let $a = \min_{\pi, y}\{ \ell(\pi, y) \, \mid \, \ell(\pi, y) \neq 0\}$ and $M = \max_{\pi, y}\ell(\pi, y)$. Given online learners  for $\Hcal_i^{j}$ for the 0-1 loss, our goal is to construct an online learner $\mathcal{Q}$ for $\Hcal$ w.r.t $\ell$ that enjoys sub-linear regret in $T$. Our strategy will be to construct a set of experts $\mathcal{E}$ using the online learners for $\Hcal_i^{j}$'s and run REWA using $\Ecal$ and an appropriately scaled version of $\ell$. Our proof borrows ideas from the realizable-to-agnostic online conversion from \cite{raman2023characterization} and so we use the same notation whenever possible. 

Let $(x_1, y_1), ..., (x_T, y_T) \in (\mathcal{X} \times \mathcal{Y})^T$ denote the stream of points to be observed by the online learner. We will assume an oblivious adversary and thus the stream is fixed before the game starts. A standard reduction (Chapter 4 in \cite{cesa2006prediction}) allows us to convert oblivious regret bounds to adaptive regret bounds. Since $\Hcal_{i}^j \subset \{0,1\}^{\Xcal}$ is online learnable w.r.t. $0$-$1$ loss, we are guaranteed the existence of online learners  $\Acal_i^{j}$ for $\Hcal_i^j$. 

   \textbf{Constructing Experts}. For any bitstring $b \in \{0, 1\}^T$, let $\phi: \{t \in [T]: b_t = 1\} \rightarrow \Scal_{K}$ denote a function mapping time points where $b_t = 1$ to rankings (permutations). Let $\Phi_b = \Scal_K^{\{t \in [T]: b_t = 1\}}$ denote all such functions $\phi$. For every $h \in \Hcal $, there exists a $\phi_b^h \in \Phi_b$  such that for all $t \in \{t: b_t = 1\}$, $\phi_b^h(t) = h(x_t)$. Let $|b| = |\{t \in [T]: b_t = 1\}|$. For every $b \in \{0, 1\}^T$ and $\phi \in \Phi_b$, we will define an Expert $E_{b, \phi}$. Expert $E_{b, \phi}$, formally presented in Algorithm \ref{alg:expert_sum@p}, uses $\mathcal{A}_i^j$'s to make predictions in each round. However, $E_{b, \phi}$ only updates the $\mathcal{A}_i^j$'s on those rounds where $b_t = 1$, using $\phi$ to compute a labeled instance. For every $b \in \{0, 1\}^T$, let $\mathcal{E}_b = \bigcup_{\phi \in \Phi_b} \{E_{b, \phi}\}$ denote the set of all Experts parameterized by functions $\phi \in \Phi_b$. If $b$ is the bitstring with all zeros, then $\mathcal{E}_b$ will be empty. Therefore, we will actually define $\mathcal{E}_b = \{E_0\} \cup \bigcup_{\phi \in \Phi_b} \{E_{b, \phi}\}$, where $E_0$ is the expert that never updates $\mathcal{A}_i^j$'s and only uses them for predictions in all $t \in [T]$.  Note that $1 \leq |\mathcal{E}_b| \leq (K!)^{|b|} \leq K^{K |b|}$. 

\begin{algorithm}
\caption{Expert $(b, \phi)$}
\label{alg:expert_sum@p}
\setcounter{AlgoLine}{0}
\KwIn{Independent copy of realizable learners  $\Acal_i^j$ of $\Hcal_i^j$ for each $(i, j) \in [K] \times [p]$}
\For{$t = 1,...,T$} {
    Receive example $x_t$
    
    Define a binary vote matrix $V_t \in \{0, 1\}^{K \times p}$ such that $V_t[i,j] = \Acal_i^j(x_t)$
    
    Predict $\hat{\pi}_t \in \argmin_{\pi \in \Scal_K} \langle \pi, V_t \mathbf{1}_p \rangle$

      \uIf{$b_t = 1$}{      
            Let $\pi = \phi(t)$ and for all $(i,j) \in [K] \times [p]$, update $\mathcal{A}_i^j$ by passing $(x_t, \pi_i^j)$
        }
}
\end{algorithm}

\begin{algorithm}
\caption{Agnostic Online Learner $\mathcal{Q}$ for $\Hcal$ w.r.t. $\ell$}
\label{alg:agn_lsum@p}
\setcounter{AlgoLine}{0}
\KwIn{ Parameter $0 < \beta < 1$}

Let $B \in \{0, 1\}^T$ s.t.  $B_t \overset{\text{iid}}{\sim} \text{Bernoulli}(\frac{T^{\beta}}{T})$

Construct the set of experts $\mathcal{E}_B = \{E_0\} \cup \bigcup_{\phi \in \Phi_B} \{E_{B, \phi}\}$ according to Algorithm \ref{alg:expert_sum@p}

Run REWA $\mathcal{P}$ using $\mathcal{E}_B$ and the loss function $\frac{\ell}{M}$ over the stream $(x_1, y_1), ..., (x_T, y_T)$

\end{algorithm}

Using these experts, Algorithm \ref{alg:agn_lsum@p} presents our agnostic online learner $\mathcal{Q}$ for $\mathcal{\Hcal}$ w.r.t $\ell \in \Lcal(\ell_{\text{sum}}^{@ p})$.  We now show that $\Qcal$ enjoys sub-linear regret. We highlight that there are three sources of randomness in online learner $\Qcal$, namely the randomness of sampling $B$, the internal randomness of $\Acal_i^j$'s, and the internal randomness of $\Pcal$. One may think of internal randomness as arising from the sampling step  involved in the randomized predictions. Let $A$ be the random variable associated with joint internal randomness of $\Acal_i^j$ for all $(i, j) \in [K] \times [p]$. Similarly, denote  $P$ to be the random variable associated with the internal randomness of $\Pcal$. We begin by using the guarantee of REWA. 

\textbf{REWA Guarantee}. Using Theorem 21.11 in \cite{ShwartzDavid} and the fact that $B, A$ and $P$ are mutually independent, REWA guarantees almost surely that
$$\sum_{t=1}^T \mathbbm{E}\left[\ell(\mathcal{P}(x_t), y_t)|B, A\right] \leq \inf_{E \in \mathcal{E}_B} \sum_{t=1}^T \ell(E(x_t), y_t) + M\sqrt{2T\ln(|\mathcal{E}_B|)}.$$
Taking an outer expectation gives
$$\mathbbm{E}\left[\sum_{t=1}^T \ell(\mathcal{P}(x_t), y_t)\right] \leq \mathbbm{E}\left[\inf_{E \in \mathcal{E}_B} \sum_{t=1}^T \ell(E(x_t), y_t) \right] + \mathbbm{E}\left[M\sqrt{2T\ln(|\mathcal{E}_B|)}\right].$$
Noting that $\Qcal(x_t) = \Pcal(x_t)$, we obtain
\begin{align*}
    \mathbbm{E}\left[\sum_{t=1}^T \ell(\mathcal{Q}(x_t), y_t) \right] 
    &\leq \mathbbm{E}\left[\inf_{E \in \mathcal{E}_B} \sum_{t=1}^T \ell(E(x_t), y_t) \right] + \mathbbm{E}\left[M\sqrt{2T\ln(|\mathcal{E}_B|)}\right]\\
    &\leq \mathbbm{E}\left[\sum_{t=1}^T \ell(E_{B, \phi_B^{h^{\star}}}(x_t), y_t) \right] + M\mathbbm{E}\left[\sqrt{2T\ln(|\mathcal{E}_B|)}\right].
\end{align*}
In the last step, we used the fact that for all $b \in \{0, 1\}^T$ and $h \in \mathcal{H}$, $E_{b, \phi_b^h} \in \mathcal{E}_b$. Here, $h^{\star} = \inf_{h \in \Hcal} \sum_{t=1}^T \ell(h(x_t), y_t)$ is the optimal function in hindsight. First, note that $\ln(|\Ecal_B|) \leq K |B| \ln(K) $. Using Jensen's inequality gives $\mathbbm{E}\left[\sqrt{2T\ln(|\mathcal{E}_B|)}\right] \leq \sqrt{2 T^{1+\beta} K \ln{ K}} $.  Thus, 
\begin{equation}\label{REWA_guarantee}
    \mathbbm{E}\left[\sum_{t=1}^T \ell(\mathcal{Q}(x_t), y_t) \right] \leq \underbrace{\mathbbm{E}\left[\sum_{t=1}^T \ell(E_{B, \phi_B^{h^{\star}}}(x_t), y_t) \right]}_{\text{(I)}} + M \sqrt{2 T^{1+\beta} K \ln{ K}}. 
\end{equation}

\textbf{Upperbounding (I).} It now suffices to upperbound  $\mathbbm{E}\left[\sum_{t=1}^T \ell(E_{B, \phi_B^{h^{\star}}}(x_t), y_t) \right]$. Recall that Lemma \ref{lem:subaddsum@p} gives pointwise 
\begin{equation}\label{eqn:olsuff_lemm}
   \ell(E_{B, \phi_B^{h^{\star}}}(x_t), y_t) \leq \ell(h^{\star}(x_t), y_t) + \frac{pM}{a}\, \expect_{j \sim \text{Unif([p])}}[\ell(E_{B, \phi_B^{h^{\star}}}(x_t), \text{BinRel}(h^{\star}(x_t), j))] 
\end{equation}
 where $M = \max_{\pi, y}\ell(\pi, y)$ and $a = \min_{\pi, y}\{ \ell(\pi, y) \, \mid \, \ell(\pi, y) \neq 0\}$. Note that, by definition of the constant $M$, we further get
\begin{equation*}
    \begin{split}
        \ell(E_{B, \phi_B^{h^{\star}}}(x_t), \text{BinRel}(h^{\star}(x_t), j)) &\leq M \,\indicator\{\ell(E_{B, \phi_B^{h^{\star}}}(x_t), \text{BinRel}(h^{\star}(x_t), j)) > 0\}  \\
        &= M\, \indicator\{\ell_{\text{sum}}^{@p}(E_{B, \phi_B^{h^{\star}}}(x_t), \text{BinRel}(h^{\star}(x_t), j)) > 0 \},
    \end{split}
\end{equation*}
where the equality follows from the fact that $\ell \in \Lcal(\ell_{\text{sum}}^{@p})$. 

In order to upperbound the indicator above, we need to introduce some more notations. Given the realizable online learner $\mathcal{A}_i^m$ for $(i,m) \in [K] \times [p]$, an instance $x \in \mathcal{X}$, and an ordered finite sequence of labeled examples $L \in (\mathcal{X} \times \{0,1\})^*$, let $\mathcal{A}_i^m(x|L)$ be the random variable denoting the prediction of $\mathcal{A}_i^m$ on the instance $x$ after running and updating on $L$. For any $b\in \{0, 1\}^T$, $h \in \mathcal{H}$, and $t \in [T]$, let $L_{b_{< t}}^{h}(i,m) = \{(x_{s}, h_{i}^m(x_{s})): s < t \text{ and } b_{s} = 1\}$ denote the \textit{subsequence} of the sequence of labeled instances $\{(x_{s}, h_i^m(x_{s}))\}_{s=1}^{t-1}$ where $b_s = 1$.  Then, for any $j \in [p]$, we have
\begin{equation*}
    \begin{split}
       \indicator\{\ell_{\text{sum}}^{@p}(E_{B, \phi_B^{h^{\star}}}(x_t), \text{BinRel}(h^{\star}(x_t), j)) > 0 \} 
       &\leq \sum_{i=1}^K \sum_{m=1}^p \indicator\{\Acal_i^{m } (x_t \mid L_{B_{< t}}^{h^{\star}}(i,m)) \neq h_i^{\star,m}(x_t)\}. 
    \end{split}
\end{equation*}
To prove the inequality above, consider the case when  $\sum_{i=1}^K \sum_{m=1}^p \indicator\{\Acal_i^{m } (x_t \mid L_{B_{< t}}^{h^{\star}}(i,m)) \neq h_i^{\star,m}(x_t)\} =0$ because the inequality is trivial otherwise. Then, we must have $\Acal_i^{m } (x_t \mid L_{B_{< t}}^{h^{\star}}(i,m)) = h_i^{ \star,m}(x_t)$ for all $(i, m) \in [K] \times [p]$. Let $V_t \in \{0, 1\}^{K \times p}$ be a binary vote matrix that $E_{B, \phi_B^{h^{\star}}}$ constructs in round $t$. Then, we have  $V_t[i, m] = \Acal_i^{m } (x_t \mid L_{B_{< t}}^{h^{\star}}(i,m)) = h_i^{ \star,m}(x_t) $ for all $(i, m) \in [K] \times [p]$.  Since $h^{\star}(x_t)$ is a permutation, the vote vector $V_t \mathbf{1}_p$ must  contain $p$ labels with distinct number of non-zero votes, namely $p, p-1, p-2,  \ldots,2, 1$ votes. Similarly, there must be $K-p$ labels with exactly $0$ votes. Thus, every $\hat{\pi}_t \in \argmin_{\pi \in \Scal_K} \langle \pi, V_t\mathbf{1}_p \rangle$ must rank label that obtained $p$ votes as $1$, label with $p-1$ votes as $2$, and so forth.  In other words, we must have $\hat{\pi}_t \stackrel{\mathclap{[p]}}{=} h^{\star}(x_t)$, and thus $\ell_{\text{sum}}^{@p}(\hat{\pi}_t, \text{BinRel}(h^{\star}(x_t), j)) =0 $ for any $j \in  [p]$ by definition of $ \ell_{\text{sum}}^{@p}$. Our claim now follows because $E_{B, \phi_B^{h^{\star}}}(x_t ) \in \argmin_{\pi \in \Scal_K} \langle \pi, V_t\mathbf{1}_p \rangle$. Using these two inequalities in equation \eqref{eqn:olsuff_lemm}, we obtain
\[\ell(E_{B, \phi_B^{h^{\star}}}(x_t), y_t) \leq \ell(h^{\star}(x_t), y_t) +\frac{pM^2}{a}\:\sum_{i=1}^K \sum_{m=1}^p \indicator\{\Acal_i^{m } (x_t \mid L_{B_{< t}}^{h^{\star}}(i,m)) \neq h_i^{\star,m}(x_t)\}, \]
which further implies that 
\[
    \mathbb{E}\left[\sum_{t=1}^T \ell(E_{B, \phi_B^{h^{\star}}}(x_t), y_t) \right] \leq \sum_{t=1}^T\ell(h^{\star}(x_t), y_t) + \frac{pM^2}{a}\:  \sum_{i=1}^K \sum_{m=1}^p \underbrace{\expect \left[ \sum_{t=1}^T \indicator\{\Acal_i^{m } (x_t \mid L_{B_{< t}}^{h^{\star}}(i,m)) \neq h_i^{\star,m}(x_t)\} \right]}_{\text{(II)}}.
\]

The first term above is the cumulative loss of the best-fixed hypothesis in hindsight. 

\textbf{Upperbounding (II).} It now suffices to show that $\expect \left[ \sum_{t=1}^T \indicator\{\Acal_i^{m } (x_t \mid L_{B_{< t}}^{h^{\star}}(i,m)) \neq h_i^{\star,m}(x_t)\} \right]$ is sub-linear for every $(i,m) \in [K] \times [p]$.  Note that we can write
\begin{equation*}
    \begin{split}
        \expect \left[ \sum_{t=1}^T \indicator\{\Acal_i^{m } (x_t \mid L_{B_{< t}}^{h^{\star}}(i,m)) \neq h_i^{\star,m}(x_t)\} \right] &=    \sum_{t=1}^T  \expect \left[ \indicator\{\Acal_i^{m } (x_t \mid L_{B_{< t}}^{h^{\star}}(i,m)) \neq h_i^{\star,m}(x_t)\}\right] \frac{\mathbb{E}\left[\indicator\{B_t = 1 \}\right]}{\mathbb{E}\left[\indicator\{B_t = 1 \} \right]}\\
        &= \frac{T}{T^{\beta}}\,  \sum_{t=1}^T \expect \left[\indicator\{\Acal_i^{m } (x_t \mid L_{B_{< t}}^{h^{\star}}(i,m)) \neq h_i^{\star,m}(x_t)\} \indicator\left\{B_t = 1 \right\}\right],
    \end{split}
\end{equation*}
where the last equality follows because $\mathbb{E}\left[\indicator\{B_t = 1 \} \right] = \frac{T^{\beta}}{T}$ and the prediction of $\Acal_i^{m } (x_t \mid L_{B_{< t}}^{h^{\star}}(i,m))$ on round $t$ only depends on bitstring ($B_1, \ldots, B_{t-1}$), but is independent of $B_t$. Next, we can use the regret guarantee of algorithm $\Acal_i^{m}$ on the rounds it was updated. That is, 
 \begin{equation*}
    \begin{split}
     \sum_{t=1}^T \expect \left[\Acal_i^{m } (x_t \mid L_{B_{< t}}^{h^{\star}}(i,m))\indicator\left\{B_t = 1 \right\}\right] &=     \expect \left[ \sum_{t  :  B_t =1} \Acal_i^{m } (x_t \mid L_{B_{< t}}^{h^{\star}}(i,m)) \neq h_i^{\star, m}(x_t)\} \right] \\
     &=     \expect \left[ \expect  \left[ \sum_{t  :  B_t =1} \Acal_i^{m } (x_t \mid L_{B_{< t}}^{h^{\star}}(i,m)) \neq h_i^{\star, m}(x_t)\} \right]  \Bigg| B \right] \leq \expect\left[R_i^m(|B|) \right] ,
    \end{split}
\end{equation*}
where $R_i^m(|B|)$ is the regret of $\mathcal{A}_i^m$, a sub-linear function of $|B|$. In the last step, we use the fact that $\Acal_i^{m}$ is a realizable algorithm for $\Hcal_i^m$ and the feedback that the algorithm received was $(x_t, h_i^{\star,m}(x_t))$ in the rounds whenever $B_t =1$. Without loss of generality, assume that $R_i^m(|B|)$ is a concave function of $|B|$. Otherwise, by Lemma 5.17 from \cite{woess2017groups}, there exists a concave sub-linear function $\tilde{R}_i^m(|B|)$ that upperbounds $R_i^m(|B|)$. By Jensen's inequality, $\mathbbm{E}_B\left[ R_i^m(|B|) \right] \leq  R_i^m(\mathbbm{E}_B\left[|B|\right])\leq  R_i^m(T^{\beta})$, a sub-linear function of $T^{\beta}$. 

Combining (I) and (II) together, we obtain
\begin{equation*}
    \begin{split}
        \mathbbm{E}\left[\sum_{t=1}^T \ell(\mathcal{Q}(x_t), y_t) \right] &\leq  \inf_{h \in \Hcal} \sum_{t=1}^T \ell(h(x_t), y_t) + \frac{pM^2}{a} \sum_{i=1}^K \sum_{m=1}^p \frac{T}{T^{\beta}}\, R_i^{m}(T^{\beta})+ M \sqrt{2 T^{1+\beta} K \ln{ K}}.\\
    \end{split}
\end{equation*}
 Since $R_i^{m}(T^{\beta})$ is a sublinear function of $T^{\beta}$, $\frac{T}{T^{\beta}}R_i^{m}(T^{\beta}) $ is a sublinear function of $T$. As the sum of sublinear functions is sublinear, the second term above must be a sublinear function of $T$. The regret is sub-linear for any choice of $\beta \in (0,1)$.  This completes our proof as we have shown that the algorithm $\Qcal$ achieves sub-linear regret in $T$.
\end{proof}

The proof of the necessity direction of Theorem \ref{thm:ol_sum@p} also involves  constructing experts and running the REWA algorithm.  Since the argument is similar, we defer details to Appendix \ref{appdx:necc_thm:ol_sum@p}.

\section{Discussion}
In this paper, we characterize the learnability of a multilabel ranking hypothesis class in both the batch and online setting for a wide range of practical ranking losses. In all cases, we show that a ranking hypothesis class is learnable if and only if a sufficient number of its binary-valued threshold restrictions are learnable. While we give explicit bounds on the sample complexity and regret, we leave it open to make them tighter for specific losses in the batch and online settings respectively.

\bibliographystyle{plainnat}
\bibliography{references}

\newpage

\appendix

\section{Categorizing Popular Ranking Losses} \label{app:catrankloss}

\begin{table}[h!]
\caption{Categorizing Popular Ranking Losses.\label{tab:rankloss}}
\centering
 \begin{tabular}{||c | c||} 
 \hline
 Loss & Loss Family  \\ [0.5ex] 
 \hline\hline
 Sum Loss@p & $\mathcal{L}(\ell^{@p}_{\text{sum}})$ \\ 
 \hline
 Precision Loss@p & $\mathcal{L}(\ell^{@p}_{\text{prec}})$ \\
  \hline
 Average Precision & $\mathcal{L}(\ell^{@K}_{\text{sum}})$ \\
  \hline
 Area Under the Curve & $\mathcal{L}(\ell^{@K}_{\text{sum}})$  \\
  \hline
 Reciprocal Rank & $\mathcal{L}(\ell^{@1}_{\text{prec}})$  \\
  \hline
 Pairwise Rank Loss & $\mathcal{L}(\ell^{@K}_{\text{sum}})$  \\
  \hline
 Discounted Cumulative Loss & $\mathcal{L}(\ell^{@K}_{\text{sum}})$  \\
  \hline
 Discounted Cumulative Loss@p & $\mathcal{L}(\ell^{@p}_{\text{sum}})$  \\ [1ex] 
 \hline
 \end{tabular}
\end{table}

In this section, we show that our loss families $\mathcal{L}(\ell_{\text{sum}}^{@p})$ and $\mathcal{L}(\ell_{\text{prec}}^{@p})$ are general and capture many of the popular ranking loss functions used in practice. We summarize the results in Table \ref{tab:rankloss}.

Recall that 
$$\mathcal{L}(\ell^{@p}_{\text{sum}}) = \{\ell \in \mathbbm{R}^{\mathcal{S}_K\times \mathcal{Y}}: \ell = 0 \text{ iff } \ell_{\text{sum}}^{@p} = 0\} \cap \{\ell \in \mathbbm{R}^{\mathcal{S}_K \times \mathcal{Y}}: \pi \stackrel{[p]}{=} \hat{\pi} \implies \ell(\pi, y) = \ell(\hat{\pi}, y) \},$$
where 
$$\ell^{@p}_{\text{sum}}(\pi, y) = \sum_{i=1}^K \min(\pi_i, p+1) y^i - Z^p_y.$$
Note that the normalization constant is defined as  $Z_y^p := \min_{\pi \in \Scal_K} \sum_{i=1}^K \min(\pi_i, p+1) y^i $
and thus only depends on $y$. Furthermore,
$$\mathcal{L}(\ell^{@p}_{\text{prec}}) = \{\ell \in \mathbbm{R}^{\mathcal{S}_K \times \mathcal{Y}}: \ell = 0 \text{ iff } \ell_{\text{prec}}^{@p} = 0\} \cap \{\ell \in \mathbbm{R}^{\mathcal{S}_K \times \mathcal{Y}}: \pi \stackrel{p}{=} \hat{\pi} \implies \ell(\pi, y) = \ell(\hat{\pi}, y) \}.$$
where
$$\ell^{@p}_{\text{prec}}(\pi, y) = Z^p_y - \sum_{i=1}^K\mathbbm{1}\{\pi_i \leq p\} y^i.$$
As before, the normalization constant $Z_y^p : = \max_{\pi \in \Scal_K}\sum_{i=1}^K\mathbbm{1}\{\pi_i \leq p\} y^i $ only depends on $y$.

In ranking literature, many evaluation metrics are often stated in terms of \textit{gain} functions. However, these can be easily converted into loss functions by subtracting the gain from the maximum possible value of the gain. When relevance scores are restricted to be binary (i.e. $\mathcal{Y} = \{0, 1\}^K$), the \textbf{Average Precision} (AP) metric is a \textit{gain} function defined as 

$$\text{AP}(\pi, y) = \frac{1}{\norm{y}_1} \sum_{i \in \{\pi_m: y^m = 1\}}\frac{\sum_{j=1}^K \mathbbm{1}\{\pi_j \leq i\}y^j}{i}.$$
Since the maximum value AP can take is $1$, we can define its loss function variant as: 
$$\ell_{\text{AP}}(\pi, y) = 1 - \text{AP}(\pi, y).$$
Note that $\ell_{\text{AP}}(\pi, y) = 0$ if and only if $\pi$ ranks all labels where $y_i = 1$ in the top $\norm{y}_1$. Therefore, $\ell_{\text{AP}}(\pi, y) \in \mathcal{L}(\ell_{\text{sum}}^{@K})$. 

Another useful metric for binary relevance feedback is the \textbf{Area Under the Curve} (AUC) loss function:

$$\ell_{\text{AUC}}(\pi, y) = \frac{1}{\norm{y}_1(K - \norm{y}_1)} \sum_{i=1}^K \sum_{j=1}^K \mathbbm{1}\{\pi_i < \pi_j\}\mathbbm{1}\{y^i < y^j\}.$$

The AUC computes the fraction of ``bad pairs" of labels (i.e those pairs of labels where $i$ was more relevant than $j$, but $i$ was ranked lower than $j$). Again, note that $\ell_{\text{AUC}}(\pi, y) = 0$ if and only if $\pi$ ranks all labels where $y^i = 1$ in the top $\norm{y}_1$. Therefore, $\ell_{\text{AP}}(\pi, y) \in \mathcal{L}(\ell_{\text{sum}}^{@K})$. 

Lastly, the \textbf{Reciprocal Rank} (RR) metric is another important \textit{gain} function for binary relevance score feedback, 
$$\text{RR}(\pi, y) = \frac{1}{\min_{i:y^i = 1}\pi_i}.$$
Its loss equivalent can be written as:
$$\ell_{\text{RR}}(\pi, y) =  1 - \text{RR}(\pi, y). $$
Since $\ell_{\text{RR}}(\pi, y)$ only cares about the relevance of the top-ranked label, we have that $\ell_{\text{RR}}(\pi, y) \in \mathcal{L}(\ell_{\text{prec}}^{@1}).$

Moving onto non-binary relevance scores, we start with the \textbf{Pairwise Rank Loss} (PL):
$$\ell_{\text{PL}}(\pi, y) = \sum_{i=1}^K \sum_{j=1}^K \mathbbm{1}\{\pi_i < \pi_j\}\mathbbm{1}\{y^i < y^j\}.$$
The Pairwise Ranking loss is the analog of AUC for non-binary relevance scores and thus $\ell_{\text{PL}}(\pi, y) \in \mathcal{L}(\ell_{\text{sum}}^{@K})$. 

Finally, we have the \textbf{Discounted Cumulative Gain} (DCG) metric, defined as: 
$$\text{DCG}(\pi, y) = \sum_{i=1}^K \frac{2^{y^i} - 1}{\log_2(1 + \pi_i)}.$$
For an appropriately chosen normalizing constant $Z_y$, we can define its associated loss: 
$$\ell_{\text{DCG}}(\pi, y) = Z_y - \text{DCG}(\pi, y).$$
Like $\ell_{\text{sum}}^{@K}$, $\ell_{\text{DCG}}(\pi, y)$ is $0$ if and only if $\pi$ ranks the $K$ labels in increasing order of relevance, breaking ties arbitrarily. Thus, $\ell_{\text{DCG}}(\pi, y) \in \mathcal{L}(\ell_{\text{sum}}^{@K})$. If one only cares about the top-$p$ ranked results, then the DCG@p loss function evaluates only the top-$p$ ranked labels: 
$$\ell^{@p}_{\text{DCG}}(\pi, y) = Z^p_y - \sum_{i=1}^K \frac{2^{y^i} - 1}{\log_2(1 + \pi_i)}\mathbbm{1}\{\pi_i \leq p\} = Z^p_y - \text{DCG}^{@p}(\pi, y).$$
Analogously, we have that $\ell^{@p}_{\text{DCG}}(\pi, y) \in \mathcal{L}(\ell_{\text{sum}}^{@p}).$ 


\section{Agnostic PAC Learnability of Score-based Rankers} \label{app:scorerankers}

In this section, we apply our results in the main paper to give sufficient conditions for the agnostic PAC learnability of score-based ranking hypothesis classes. A score-based ranking hypothesis $h: \mathcal{X} \rightarrow \mathcal{S}_K$ first maps an input $x \in \mathcal{X}$ to a vector in $\mathbbm{R}^K$ representing the ``score" for each label. Then, it outputs a ranking (permutation) over the labels in $[K]$ by sorting the real-valued vector in decreasing order of score. 

More formally, let $\mathcal{F} \subset (\mathbbm{R}^K)^{\mathcal{X}}$ denote a set of functions mapping elements from the input space $\mathcal{X}$ to score-vectors in $\mathbbm{R}^K$. For each $f \in \mathcal{F}$, define the score-based ranking hypothesis $h_f(x) = \text{argsort}(f(x))$ which first computes the score-vector $f(x) \in \mathbbm{R}^K$, and then outputs a ranking by sorting $f(x)$ in decreasing order, breaking ties by giving the smaller label the higher rank. That is, if $f_1(x) = f_2(x)$, then label $1$ will be ranked higher than label $2$. Given $\mathcal{F}$, define its induced score-based ranking hypothesis class as $\mathcal{H} = \{h_f: f \in \mathcal{F}\}$. Since our characterization of ranking learnability relates the learnability of $\mathcal{H}$ to the learnability of the \textit{binary} threshold-restricted classes $\mathcal{H}_i^j = \{h_i^j: h \in \mathcal{H}\}$, it suffices to consider an arbitrary threshold-restricted class $\mathcal{H}_i^j$ and bound its VC dimension. Before we do so, we need some more notation regarding $\mathcal{F}$.

For each $k \in [K]$, define the scalar-valued function class $\Fcal_k = \{f_k \mid (f_1, \ldots, f_K) \in \Fcal\}$ by restricting each function in $\Fcal$ to its $k^{th}$ coordinate output. Here, each $\Fcal_k \subset \mathbbm{R}^{\Xcal}$ and we can write $\Fcal = (\Fcal_1, \ldots, \Fcal_K)$. For a function $f \in \Fcal$, we will use $f_k(x)$ to denote the $k^{th}$ coordinate output of $f(x)$. For every $(i, j) \in [K] \times [K]$, define the function class $\mathcal{F}_i - \mathcal{F}_j  = \{f_i - f_j: f \in \mathcal{F}\}$ where we let $f_i - f_j: \mathcal{X} \rightarrow \mathbbm{R}$ denote a function such that $(f_i - f_j)(x) = f_i(x) - f_j(x).$ Subsequently, for any $(i, j) \in [K] \times [K]$, define the \textit{binary} hypothesis classes $\mathcal{G}_{i,j} = \{\mathbbm{1}\{(f_i - f_j)(x) < 0\}: f_i - f_j \in \mathcal{F}_i - \mathcal{F}_j\}$ and $\tilde{\mathcal{G}}_{i,j} = \{\mathbbm{1}\{(f_i - f_j)(x) \leq 0\}: f_i - f_j \in \mathcal{F}_i - \mathcal{F}_j\}$. Finally, let $C_j: \{0, 1\}^K \rightarrow \{0, 1\}$ be the $K$-wise composition s.t. $C_j(b) = \mathbbm{1}\{\sum_{i=1}^K b_i \leq j\}$ and define $C_j(\mathcal{G}_{1}, ..., \mathcal{G}_{K}) = \{C_j(g_{1}, ..., g_{K}): (g_{1}, ..., g_K) \in \mathcal{G}_{1}\times ...\times \mathcal{G}_{K} \}.$ In other words, $C_j(\mathcal{G}_{1}, ..., \mathcal{G}_{K})$ is the \textit{binary} hypothesis class constructed by taking all combinations of binary classifiers from  $\mathcal{G}_{1}, ..., \mathcal{G}_{K}$, summing them up, and thresholding the sum at $j$. We are now ready to bound the VC dimension of an arbitrary threshold-restricted class $\mathcal{H}_i^j$.

 Consider an arbitrary threshold-restricted class $\mathcal{H}_i^j$ and hypothesis $h \in \mathcal{H}$. By definition, $h_i^j \in \mathcal{H}_i^j$. Let $f \in \mathcal{F}$ denote the function associated with $h$. Given an instance $x \in \mathcal{X}$, recall that $h_i^j(x) = \mathbbm{1}\{h_i(x) \leq j\}$ where $h_i(x)$ is the rank that $h$ gives to the label $i$ for instance $x$. Since $h(x) = \text{argsort}(f(x))$, we have 

\begin{align*}
h_i(x) &= \text{argsort}(f(x))[i]\\
&= \sum_{m=1}^i \mathbbm{1}\{f_i(x) \leq f_m(x)\} + \sum_{m=i+1}^K \mathbbm{1}\{f_i(x) < f_m(x)\}\\
&= \sum_{m=1}^i \mathbbm{1}\{(f_i - f_m)(x)\leq 0\} + \sum_{m=i+1}^K \mathbbm{1}\{(f_i - f_m)(x) < 0\}
\end{align*}


Thus, we can write: 


$$h_i^j(x) = \mathbbm{1}\left\{\left(\sum_{m=1}^i \mathbbm{1}\{(f_i - f_m)(x)\leq 0\} + \sum_{m=i+1}^K \mathbbm{1}\{(f_i - f_m)(x) < 0\}\right) \leq j \right\}.$$

Note that $h_i^j \in C_j(\tilde{\mathcal{G}}_{i,1}, ...,\tilde{\mathcal{G}}_{i,i},\mathcal{G}_{i,i+1}, ... ,\mathcal{G}_{i,K})$ by construction. Since $h$, and therefore $h_i^j$, was arbitrary, it further follows that $\mathcal{H}_i^j \subset C_j(\tilde{\mathcal{G}}_{i,1}, ...,\tilde{\mathcal{G}}_{i,i},\mathcal{G}_{i,i+1}, ... ,\mathcal{G}_{i,K})$. Therefore, 

$$\text{VC}(\mathcal{H}_i^j) \leq \text{VC}(C_j(\tilde{\mathcal{G}}_{i,1}, ...,\tilde{\mathcal{G}}_{i,i},\mathcal{G}_{i,i+1}, ... ,\mathcal{G}_{i,K})).$$

Since $C_j(\tilde{\mathcal{G}}_{i,1}, ...,\tilde{\mathcal{G}}_{i,i},\mathcal{G}_{i,i+1}, ... ,\mathcal{G}_{i,K})$ is some $K$-wise composition of binary classes  $\tilde{\mathcal{G}}_{i,1}, ...,\tilde{\mathcal{G}}_{i,i},\mathcal{G}_{i,i+1}, ... ,\mathcal{G}_{i,K}$, standard VC composition guarantees that $\text{VC}(C_j(\tilde{\mathcal{G}}_{i,1}, ...,\tilde{\mathcal{G}}_{i,i},\mathcal{G}_{i,i+1}, ... ,\mathcal{G}_{i,K})) = \tilde{O}(\text{VC}(\tilde{\mathcal{G}}_{i,1})+...+\text{VC}(\tilde{\mathcal{G}}_{i,i}) + \text{VC}(\mathcal{G}_{i,i+1})+...+\text{VC}(\mathcal{G}_{i,K}))$, where we hide log factors of $K$ and the VC dimensions \citep{dudley1978central, alon2020closure}. Putting things together, we have that 
$$\text{VC}(\mathcal{H}_i^j) \leq \tilde{O}(\text{VC}(\tilde{\mathcal{G}}_{i,1})+...+\text{VC}(\tilde{\mathcal{G}}_{i,i}) + \text{VC}(\mathcal{G}_{i,i+1})+...+\text{VC}(\mathcal{G}_{i,K})).$$
An identical analysis can also be used to give sufficient conditions for the \textit{online} learnability of score-based rankers in terms of the Littlestone dimensions of $\mathcal{H}^{i}_j$. 

 Now, we consider the special class of \textit{linear} score-based ranker and prove Lemma \ref{lem:example}.
 
 \begin{proof} (of Lemma \ref{lem:example})
     Let $\mathcal{X} = \mathbbm{R}^d$ and $\mathcal{F} = \{f_W: W \in \mathbbm{R}^{K \times d}\}$ s.t. $f_W(x)= Wx$. Consider the class of linear score-based rankers $\mathcal{H} = \{h_{f_W}:f_W \in \mathcal{F}\}$ where $h_{f_W}(x) = \text{argsort}(f_W(x)) = \text{argsort}(Wx)$ breaking ties in the same way mentioned above. Note for all $i \in [K]$, $\mathcal{F}_i = \{f_w: w \in \mathbbm{R}^d\}$ where $f_w(x) = w^{T}x$. Furthermore, $\mathcal{F}_i - \mathcal{F}_j = \mathcal{F}_i = \mathcal{F}_j$. Therefore, for any $(i, j) \in [K] \times [K]$, 

$$\mathcal{G}_{i,j} = \{\mathbbm{1}\{(f_i - f_j)(x) < 0\}: f_i - f_j \in \mathcal{F}_i - \mathcal{F}_j\} = \{\mathbbm{1}\{f_w(x) < 0\}: w \in \mathbbm{R}^d\}$$

and 

$$\tilde{\mathcal{G}}_{i,j} = \{\mathbbm{1}\{(f_i - f_j)(x) \leq 0\}: f_i - f_j \in \mathcal{F}_i - \mathcal{F}_j\} = \{\mathbbm{1}\{f_w(x) \leq 0\}: w \in \mathbbm{R}^d\}$$

are the set of half-space classifiers passing through the origin with dimension $d$. Since for all $(i, j) \in [K] \times [K]$, $\text{VC}(\tilde{\mathcal{G}}_{i,j}) = \text{VC}(\mathcal{G}_{i,j}) = d$, we get that $\text{VC}(\mathcal{H}_i^j) \leq \tilde{O}(Kd).$
 \end{proof}

\section{Proofs for Batch Multilabel Ranking }
\label{app:batchproofs}

Since many of the ranking losses we consider map to values in $\mathbbm{R}$,  the \textit{empirical} Rademacher complexity will be a useful tool for proving learnability in the batch setting. 

\begin{definition} [Empirical Rademacher Complexity of Loss Class] Let $\ell(\cdot, \cdot)$ be a loss function, $S = \{(x_1, y_1), ..., (x_n, y_n)\} \in (\mathcal{X} \times \mathcal{Y})^*$ be a set of examples, and $\ell \circ \mathcal{\mathcal{H}} = \{(x, y) \mapsto \ell(h(x), y): h \in \mathcal{H}\}$ be a loss class. The empirical Rademacher complexity of $\ell \circ \mathcal{H}$ is defined as 
$$\hat{\mathfrak{R}}_n(\ell \circ \mathcal{\mathcal{H}}) = \mathbbm{E}_{\sigma}\left[\sup_{h \in \mathcal{H}} \left(\frac{1}{n}\sum_{i=1}^n \sigma_i \ell(h(x_i), y_i) \right)\right]$$
where $\sigma_1, ..., \sigma_n$ are independent \emph{Rademacher} random variables. 
\end{definition}

In particular, a standard result relates the empirical Rademacher complexity to the generalization error of hypotheses in $\mathcal{H}$ with respect to a real-valued bounded loss function $\ell(h(x), y)$ \citep{bartlett2002rademacher}.

\begin{proposition}[Rademacher-based Uniform Convergence]
\label{prop:rad}
 Let $\mathcal{D}$ be a distribution over $\mathcal{X} \times \mathcal{Y}$ and $\ell(\cdot, \cdot) \leq c$ be a bounded loss function. With probability at least $1 - \delta$ over the sample $S \sim \mathcal{D}^n$, for all $h \in \mathcal{H}$ simultaneously, 
$$\left|\mathbbm{E}_\mathcal{D}[\ell(h(x), y)] - \hat{\mathbbm{E}}_S[\ell(h(x), y)]\right| \leq  2\hat{\mathfrak{R}}_n(\mathcal{F}) + O\left(c\sqrt{\frac{\ln(\frac{1}{\delta})}{n}}\right)$$
where $\hat{\mathbbm{E}}_S[\ell(h(x), y)] = \frac{1}{|S|}\sum_{(x, y) \in S} \ell(h(x), y)$ is the empirical average of the loss over $S$.
\end{proposition}

When the empirical Rademacher complexity of the loss class $\ell \circ \mathcal{H} = \{(x, y) \mapsto \ell(h(x), y): h \in \mathcal{H}\}$ is $o(1)$, we state that $\mathcal{H}$ enjoys the uniform convergence property w.r.t $\ell$. If $\mathcal{H}$ enjoys the uniform convergence property w.r.t.\ a loss $\ell$, a standard result shows that $\mathcal{H}$ is learnable according to Definition \ref{def:batch_learn} via Empirical Risk Minimization (ERM) (Theorem 26.5 in \cite{ShwartzDavid}).

\subsection{Proof of Lemma \ref{lem:batch_suffsumloss}}
\begin{proof} 
Let $\mathcal{H} \subset \mathcal{S}_K^\mathcal{X}$ be an arbitrary ranking hypothesis class. We need to show that if $\mathcal{H}_i^j$ is agnostic PAC learnable w.r.t to 0-1 loss for all $(i, j) \in [K] \times [p]$, then ERM is an agnostic PAC learnable w.r.t $\ell_{\text{sum}}^{@p}$. By Proposition \ref{prop:rad}, it suffices to show that the empirical Rademacher complexity of the loss class $\ell_{\text{sum}}^{@p} \circ \mathcal{H}$ vanishes as $n$ increases. This will imply that $\ell_{\text{sum}}^{@p}$ enjoys the uniform convergence property, and therefore ERM is an agnostic PAC learner for $\mathcal{H}$ w.r.t $\ell_{\text{sum}}^{@p}$. By definition, we have that 

\begin{align*}
\hat{\mathfrak{R}}_n(\ell_{\text{sum}}^{@p} \circ \mathcal{H}) &= \mathbbm{E}_{\sigma \sim \{\pm1\}^n}\left[\sup_{h \in \mathcal{H}}\frac{1}{n}\sum_{i=1}^n \sigma_i \ell_{\text{sum}}^{@p}(h(x_i), y_i)) \right]\\
&= \mathbbm{E}_{\sigma \sim \{\pm1\}^n}\left[\sup_{h \in \mathcal{H}}\frac{1}{n}\sum_{i=1}^n \left(\sum_{m=1}^K \sigma_i  \min(h_m(x_i), p+1)y^m_i - \sigma_i Z_{y_i}^p\right) \right]\\
&= \mathbbm{E}_{\sigma \sim \{\pm1\}^n}\left[\sup_{h \in \mathcal{H}}\frac{1}{n}\sum_{i=1}^n \sum_{m=1}^K \sigma_i  \min(h_m(x_i), p+1)y^m_i \right]\\
&\leq \sum_{m=1}^K \mathbbm{E}_{\sigma \sim \{\pm1\}^n}\left[\sup_{h \in \mathcal{H}}\frac{1}{n}\sum_{i=1}^n \sigma_i  \min(h_m(x_i), p+1)y^m_i \right] \\
&\leq B\sum_{m=1}^K \mathbbm{E}_{\sigma \sim \{\pm1\}^n}\left[\sup_{h \in \mathcal{H}}\frac{1}{n}\sum_{i=1}^n \sigma_i  \min(h_m(x_i), p+1) \right]\\
\end{align*}

where the second inequality follows from the fact that $y_i^m \leq B$ and Talagrand's Contraction Lemma \cite{ledoux1991probability}. 

Next note that $\min(h_m(x_i), p+1) = (p+1) - \sum_{j=1}^p \mathbbm{1}\{h_m(x_i) \leq j\} = (p+1) - \sum_{j=1}^p h_m^j(x_i)$. Substituting and getting rid of constant factors, we have that 

\begin{align*}
\hat{\mathfrak{R}}_n(\ell^{@p}_{\text{sum}} \circ \mathcal{H}) &\leq  B\sum_{m=1}^K \mathbbm{E}_{\sigma \sim \{\pm1\}^n}\left[\sup_{h_m \in \mathcal{H}_m}\frac{1}{n}\sum_{i=1}^n \sigma_i \sum_{j=1}^p h_m^j(x_i) \right]\\
&\leq B\sum_{m=1}^K \sum_{j=1}^p \mathbbm{E}_{\sigma \sim \{\pm1\}^n}\left[\sup_{h_m \in \mathcal{H}_m}\frac{1}{n}\sum_{i=1}^n \sigma_i h_m^j(x_i) \right]\\
&= B\sum_{m=1}^K \sum_{j=1}^p \hat{\mathfrak{R}}_n(\mathcal{H}_m^j).
\end{align*}

Since for $\mathcal{H}_m^j$ is agnostic PAC learnable w.r.t 0-1 loss, by Theorem 6.5 in \cite{ShwartzDavid}, $\lim_{n \rightarrow \infty} \hat{\mathfrak{R}}_n(\mathcal{H}_m^j) = 0$. Since $p, K$ and $B$ are finite, 

$$\lim_{n \rightarrow \infty} \hat{\mathfrak{R}}_n(\ell^{@p}_{\text{sum}} \circ \mathcal{H}) =  \lim_{n \rightarrow \infty}B\sum_{m=1}^K \sum_{j=1}^p \hat{\mathfrak{R}}_n(\mathcal{H}_m^j) = 0$$.

 By Proposition \ref{prop:rad}, this implies that $\ell_{\text{sum}}^{@p}$ enjoys the uniform convergence property, and therefore ERM using $\ell_{\text{sum}}^{@p}$ is an agnostic PAC learner for $\mathcal{H}$.
\end{proof}

\subsection{Proof of Lemma \ref{lem:batch_sum@pnec}}
\begin{proof}
    Fix $\ell \in \mathcal{L}(\ell^{@p}_{\text{sum}})$ and $(i, j) \in [K] \times [p]$. Let $a = \min_{\pi, y}\{ \ell(\pi, y) \, \mid \, \ell(\pi, y) \neq 0\}$.  Let $\mathcal{H}$ be an arbitrary ranking hypothesis class and $\mathcal{A}$ be an agnostic PAC learner for $\mathcal{H}$ w.r.t $\ell$. Our goal will be to use $\mathcal{A}$ to construct an agnostic PAC learner for $\mathcal{H}_i^j$. 
    
     Let $\mathcal{D}$ be distribution over $\mathcal{X} \times \{0, 1\}$  and $h^{\star, j}_i = \argmin_{h_i^j \in \mathcal{H}_i^j}\mathbbm{E}_{\mathcal{D}}\left[\mathbbm{1}\{h_i^j(x) \neq y\} \right]$ be the optimal hypothesis. Let $h^{\star} \in \mathcal{H}$ be any valid completion of $h^{\star, j}_i$. Our goal will be to show that Algorithm \ref{alg:batch_realHij} is an agnostic PAC learner for $\mathcal{H}_i^j$ w.r.t 0-1 loss. 

\begin{algorithm}
\caption{Agnostic  PAC learner  for $\Hcal_i^j$ w.r.t. 0-1 loss}
\label{alg:batch_realHij}
\setcounter{AlgoLine}{0}
\KwIn{Agnostic PAC  learner $\Acal$ for $\Hcal$ w.r.t $\ell$, unlabeled samples $S_U \sim \Dcal_{\Xcal}^n$, and labeled samples $S_L \sim \Dcal^m$ }

For each $h \in \mathcal{H}_{|S_U}$, construct a dataset
$$S_U^h = \{(x_1, \tilde{y}_1), ..., (x_n,\tilde{y}_n)\} \text{ s.t. } \tilde{y}_i = \text{BinRel}(h(x_i), j)$$

Run $\Acal$ over all datasets to get $C(S_U) := \left\{\Acal\big(S_U^h \big) 
 \mid  h \in \mathcal{H}_{|S_U}\right\}$

Define $C_i^j(S_U) = \{g_i^j| g \in C(S_U)\}$

Return $\hat{g}_i^j \in C_i^j(S_U)$  with the lowest empirical error over $S_L$ w.r.t. 0-1 loss. 
\end{algorithm}

 Consider the sample $S_U^{h^{\star}}$ and let $g = \mathcal{A}(S_U^{h^{\star}})$. We can think of $g$ as the output of $\mathcal{A}$ run over an i.i.d sample $S$ drawn from $\mathcal{D}^{\star}$,  a joint distribution over $\mathcal{X} \times \mathcal{Y}$ defined procedurally by first sampling $x \sim \mathcal{D}_{\mathcal{X}}$ and then outputting the labeled sample $(x, \text{BinRel}(h^{\star}(x), j))$. 
Note that $\mathcal{D}^{\star}$ is a realizable distribution (realized by $h^{\star}$) w.r.t $\ell_{\text{sum}}^{@p}$ and therefore also $\ell$. Let $m_{\Acal}(\epsilon, \delta, K)$ be the sample complexity of $\mathcal{A}$. Since $\mathcal{A}$ is an agnostic PAC learner for $\mathcal{H}$ w.r.t $\ell$, we have that for sample size $n \geq m_{\Acal}(\frac{a\epsilon}{2}, \delta/2, K)$, with probability at least $1-\frac{\delta}{2}$,

$$\mathbbm{E}_{\mathcal{D}^{\star}}\left[\ell(g(x), y) \right] \leq \inf_{h \in \mathcal{H}}\mathbbm{E}_{\mathcal{D}^{\star}}\left[\ell(h(x), y) \right] + \frac{a\epsilon}{2} = \frac{a\epsilon}{2}.$$

Furthermore, by definition of $\mathcal{D}^{\star}$, $\mathbbm{E}_{\mathcal{D}^{\star}}\left[\ell(g(x), y) \right] = \mathbbm{E}_{x \sim \mathcal{D}_{\mathcal{X}}}\left[\ell(g(x), \text{BinRel}(h^{\star}(x), j)) \right]$. Therefore, $\mathbbm{E}_{x \sim \mathcal{D}_{\mathcal{X}}}\left[\ell(g(x), \text{BinRel}(h^{\star}(x), j)) \right] \leq \frac{a\epsilon}{2}.$  Next, using Lemma \ref{lem:sum@plb}, we have pointwise that

\begin{equation*}
    \begin{split}
        \indicator\{g_i^j(x) \neq h_i^{ \star, j}(x)\} &\leq \indicator\{\ell_{\text{sum}}^{@p}(g(x), \text{BinRel}(h^{\star}(x), j)) > 0\} \\
        &=  \indicator\{\ell(g(x), \text{BinRel}(h^{\star}(x), j)) > 0\} \\
        &\leq \frac{1}{a} \, \ell(g(x), \text{BinRel}(h^{\star}(x), j)).
    \end{split}
\end{equation*}

Taking expectations on both sides gives, 

$$
    \mathbbm{E}_{\mathcal{D}}\left[  \indicator\{g_i^j(x) \neq h_i^{ \star, j}(x)\}\right] \leq \frac{1}{a} \,\mathbbm{E}_{\mathcal{D}}\left[ \ell(g(x), \text{BinRel}(h^{\star}(x), j)) \right] \leq \frac{\epsilon}{2},
$$

where in the last inequality we use the fact that $\mathbbm{E}_{x \sim \mathcal{D}_{\mathcal{X}}}\left[\ell(g(x), \text{BinRel}(h^{\star}(x), j)) \right] \leq \frac{a\epsilon}{2}.$ Finally, using the triangle inequality, we have that 

\begin{align*}
    \mathbbm{E}_{\mathcal{D}}\left[  \indicator\{g_i^j(x) \neq y\}\right] &\leq \mathbbm{E}_{\mathcal{D}}\left[  \indicator\{h^{\star, j}_i(x) \neq y\}\right] + \mathbbm{E}_{\mathcal{D}}\left[  \indicator\{g_i^j(x) \neq h^{\star, j}_i(x\}\right] \\
    &\leq \mathbbm{E}_{\mathcal{D}}\left[  \indicator\{h^{\star, j}_i(x) \neq y\}\right] + \frac{\epsilon}{2}\\
    &= \argmin_{h_i^j \in \mathcal{H}_i^j}\mathbbm{E}_{\mathcal{D}}\left[\mathbbm{1}\{h_i^j(x) \neq y\} \right] + \frac{\epsilon}{2} .
\end{align*}

Since $g_i^j \in C_i^j(S_U)$,  we have shown that $C_i^j(S_U)$ contains a hypothesis that generalizes well w.r.t $\mathcal{D}$. Now we want to show that the predictor $\hat{g}_i^j$ returned in step 4 also generalizes well. Crucially, observe that $C_i^j(S_U)$ is a finite hypothesis class with cardinality at most  $K^{jn}$. Therefore, by standard Chernoff and union bounds, with probability at least $1-\delta/2$, the empirical risk of every hypothesis in $C_i^j(S_U)$ on a sample of size $\geq  \frac{8}{\epsilon^2} \log{\frac{4 |C_i^j(S_U)|}{\delta}}$ is at most $\epsilon/4$ away from its true error. So, if $m = |S_L| \geq  \frac{8}{\epsilon^2} \log{\frac{4 |C_i^j(S_U)|}{\delta}}$, then with probability at least $1-\delta/2$, we have
\[\frac{1}{|S_L|} \sum_{(x, y) \in S_L} \mathbbm{1}\{g_i^j(x) \neq  y  \}\leq \expect_{\Dcal}\left[\mathbbm{1}\{g_i^j(x) \neq  y  \}\right] + \frac{\epsilon}{4} \leq \frac{3\epsilon}{4}. \]
Since $\hat{g}_i^j$ is the ERM on $S_L$ over $C_i^j(S_U)$, its empirical risk can be at most $\frac{3\epsilon}{4}$. Given that the population risk of $\hat{g}_i^j$ can be at most $\epsilon/4$ away from its empirical risk, we have that
\[\expect_{\Dcal}[\mathbbm{1}\{\hat{g}_i^j(x) \neq y\}]  \leq \argmin_{h_i^j \in \mathcal{H}_i^j}\mathbbm{E}_{\mathcal{D}}\left[\mathbbm{1}\{h_i^j(x) \neq y\} \right] + \epsilon. \]

Applying union bounds, the entire process succeeds with probability $1- \delta$.  We can compute the upper bound on the sample complexity of Algorithm \ref{alg:batch_realHij}, denoted $n(\epsilon, 
\delta, K)$, as 
\begin{equation*}
    \begin{split}
       n(\epsilon, \delta, K) &\leq m_{\mathcal{A}}(\frac{a\epsilon}{2}, \delta/2, K) + O \left( \frac{1}{\epsilon^2} \log{\frac{|C(S_U)|}{\delta}}\right)\\
       &\leq  m_{\mathcal{A}}(\frac{a\epsilon}{2}, \delta/2, K) + O \left( \frac{ K m_{\mathcal{A}}(\frac{a\epsilon}{2}, \delta/2, K)\,+ \log{\frac{1}{\delta}}}{\epsilon^2} \right),
    \end{split}
\end{equation*}
where we use $|C(S_U)| \leq 2^{K m_{\mathcal{A}}(\frac{a\epsilon}{2}, \delta/2, K)}$. This shows that Algorithm \ref{alg:batch_realHij} is an agnostic PAC learner for $\mathcal{H}_i^j$ w.r.t 0-1 loss.  Since our choice of loss $\ell \in \mathcal{L}(\ell_{\text{sum}}^{@p})$ and indices $(i, j)$ were arbitrary, agnostic PAC learnability of $\mathcal{H}$ w.r.t $\ell$ implies agnostic PAC learnability of $\mathcal{H}_i^j$ w.r.t the 0-1 loss for all $(i, j) \in [K] \times [p]$. 
\end{proof}

\subsection{Characterizing Batch Learnability of $\mathcal{L}(\ell^{@p}_{\text{prec}})$} \label{app:batchprec@p}

In this section, we prove Theorem \ref{thm:batch_prec@p} which characterizes the agnostic PAC learnability of an arbitrary hypothesis class $\mathcal{H} \subset \mathcal{S}_K^{\mathcal{X}}$ w.r.t losses in $\mathcal{L}(\ell_{\text{prec}}^{@p})$. Our proof will again be in three parts. First, we will show that if for all $i \in [K]$, $\mathcal{H}_i^p$ is agnostic PAC learnable w.r.t the 0-1 loss, then ERM is an agnostic PAC learnable w.r.t $\ell_{\text{prec}}^{@p}$. Next, we show that if $\mathcal{H}$ is agnostic PAC learnable w.r.t $\ell_{\text{prec}}^{@p}$, then $\mathcal{H}$ is agnostic PAC learnable w.r.t any loss $\ell \in \mathcal{L}(\ell_{\text{prec}}^{@p})$. Finally, we prove the necessity direction - if $\mathcal{H}$ is agnostic PAC learnable w.r.t an arbitrary $\ell \in \mathcal{L}(\ell_{\text{prec}}^{@p})$, then for all $i \in [K]$, $\mathcal{H}_i^p$ is agnostic PAC learnable w.r.t the 0-1 loss. 

We begin with Lemma \ref{lem:batch_suffprecloss} which asserts that if for all $i \in [K]$, $\mathcal{H}_i^p$ is agnostic PAC learnable, then ERM is an agnostic PAC learner for $\mathcal{H}$ w.r.t $\ell_{\text{prec}}^{@p}$.

\begin{lemma} \label{lem:batch_suffprecloss}
If for all $i \in [K]$, $\mathcal{H}_i^p$ is agnostic PAC learnable w.r.t the 0-1 loss, then ERM is an agnostic PAC learner for $\mathcal{H} \subset \mathcal{S}_K^{\mathcal{X}}$  w.r.t $\ell_{\text{prec}}^{@p}$
\end{lemma}

The proof of Lemma \ref{lem:batch_suffprecloss} is similar to the proof of Lemma \ref{lem:batch_suffsumloss} and involves bounding the empirical Rademacher complexity of the loss class $\ell_{\text{prec}}^{@p} \circ \mathcal{H}$. This will imply that $\ell_{\text{prec}}^{@p}$ enjoys the uniform convergence property, and therefore ERM is an agnostic PAC learner for $\mathcal{H}$ w.r.t $\ell_{\text{prec}}^{@p}$. The key insight is that we can write $\ell_{\text{prec}}^{@p}(h(x), y) =  Z^p_y - \sum_{i=1}^K\mathbbm{1}\{h_i(x) \leq p\} y^i =  Z^p_y - \sum_{i=1}^K h_i^p(x) y^i$. Since $Z_y^p$ does not depend on $h(x)$ and $y^i \leq B$, we can upperbound the empirical Rademacher complexity in terms of the empirical Rademacher complexities of $\mathcal{H}_i^p$ using Talagrand's contraction.  

\begin{proof}
Let $\mathcal{H} \subset \mathcal{S}_K^\mathcal{X}$ be an arbitrary ranking hypothesis class. Similar to the proof of Lemma \ref{lem:batch_suffsumloss}, it suffices to show that the empirical Rademacher complexity of the loss class $\ell_{\text{prec}}^{@p} \circ \mathcal{H}$ vanishes. By Proposition \ref{prop:rad}, this will imply that $\ell_{\text{prec}}^{@p}$ enjoys the uniform convergence property, and therefore ERM is an agnostic PAC learner for $\mathcal{H}$ w.r.t $\ell_{\text{prec}}^{@p}$. By definition, we have that 
\begin{align*}
\hat{\mathfrak{R}}_n(\ell_{\text{prec}}^{@p} \circ \mathcal{H}) &= \mathbbm{E}_{\sigma \sim \{\pm1\}^n}\left[\sup_{h \in \mathcal{H}}\frac{1}{n}\sum_{i=1}^n \sigma_i \ell_{\text{prec}}^{@p}(h(x_i), y_i)) \right]\\
&= \mathbbm{E}_{\sigma \sim \{\pm1\}^n}\left[\sup_{h \in \mathcal{H}}\frac{1}{n}\sum_{i=1}^n \left(\sigma_i Z_{y_i}^p - \sum_{m=1}^K \sigma_i  \mathbbm{1}\{h_m(x_i) \leq p\}y_i^m \right) \right]\\
&= \mathbbm{E}_{\sigma \sim \{\pm1\}^n}\left[\sup_{h \in \mathcal{H}}\frac{1}{n}\sum_{i=1}^n \sum_{m=1}^K \sigma_i h_m^p(x_i)y_i^m \right]\\
&\leq \sum_{m=1}^K \mathbbm{E}_{\sigma \sim \{\pm1\}^n}\left[\sup_{h \in \mathcal{H}}\frac{1}{n}\sum_{i=1}^n \sigma_i h_m^p(x_i)y_i^m  \right]\\
&\leq B\sum_{m=1}^K \mathbbm{E}_{\sigma \sim \{\pm1\}^n}\left[\sup_{h \in \mathcal{H}}\frac{1}{n}\sum_{i=1}^n \sigma_i h_m^p(x_i) \right] \\
&= B\sum_{m=1}^K \hat{\mathfrak{R}}_n(\mathcal{H}^p_m),\\
\end{align*}
where the second inequality follows from Talagrand's Contraction Lemma and the fact that $y_i^m \leq B$ for all $i, m$. Since for all $m \in [K]$, $\mathcal{H}_m^p$ is agnostic PAC learnable w.r.t 0-1 loss, by Theorem 6.7 in \cite{ShwartzDavid}, $\lim_{n \rightarrow \infty} \hat{\mathfrak{R}}_n(\mathcal{H}_m^p) = 0$. Since $K$ and $B$ are finite, 

$$\lim_{n \rightarrow \infty} \hat{\mathfrak{R}}_n(\ell_{\text{prec}}^{@p} \circ \mathcal{H}) =  \lim_{n \rightarrow \infty}B\sum_{m=1}^K \hat{\mathfrak{R}}_n(\mathcal{H}^p_m) = 0$$.

By Proposition \ref{prop:rad}, this implies that $\ell_{\text{prec}}^{@p}$ enjoys the uniform convergence property, and therefore ERM using $\ell_{\text{prec}}^{@p}$ is an agnostic PAC learner for $\mathcal{H}$.
\end{proof}

Next, Lemma \ref{lem:batch_precloss2arbloss} extends the learnability of $\ell_{\text{prec}}^{@p}$ to the learnability of any loss $\ell \in \mathcal{L}(\ell_{\text{prec}}^{@p})$. In particular, Lemma \ref{lem:batch_precloss2arbloss} asserts that if $\mathcal{H}$ is agnostic PAC learnable w.r.t $\ell_{\text{prec}}^{@p}$ then $\mathcal{H}$ is also agnostic PAC learnable w.r.t any $\ell \in \mathcal{L}(\ell_{\text{prec}}^{@p})$.   

\begin{lemma} \label{lem:batch_precloss2arbloss}
 If a hypothesis class $\mathcal{H} \subset \mathcal{S}_K^{\mathcal{X}}$ is agnostic PAC learnable w.r.t $\ell_{\text{prec}}^{@p}$, then $\mathcal{H}$ is agnostic PAC learnable w.r.t any $\ell \in  \mathcal{L}(\ell_{\text{prec}}^{@p})$.
\end{lemma}

The proof of Lemma \ref{lem:batch_precloss2arbloss} follows the same the exact same strategy used in proving Lemma \ref{lem:batch_sumloss2arbloss}. More specifically, given an agnostic PAC learner $\mathcal{A}$ for $\mathcal{H}$ w.r.t. $\ell_{\text{prec}}^{@p}$, we first create a \textit{realizable} PAC learner for $\mathcal{H}$ w.r.t $\ell \in \mathcal{L}(\ell_{\text{prec}}^{@p})$. Then, we use a similar realizable-to-agnostic conversion technique as in the proof of Lemma \ref{lem:batch_sumloss2arbloss} to convert the realizable PAC learner into an agnostic PAC learner for $\mathcal{H}$ w.r.t $\ell$. 

\begin{proof}
Fix $\ell \in \mathcal{L}(\ell_{\text{prec}}^{@p})$.  Let $a = \min_{\pi, y}\{ \ell(\pi, y) \, \mid \, \ell(\pi, y) \neq 0\}$ and $b = \max_{\pi, y}\ell(\pi, y)$. We need to show that if $\mathcal{H}$ is agnostic PAC learnable w.r.t $\ell_{\text{prec}}^{@p}$, then $\mathcal{H}$ is agnostic PAC learnable w.r.t $\ell$. We will do so in two steps. First, we will show that if $\mathcal{A}$ is an agnostic PAC learner for $\mathcal{H}$ w.r.t. $\ell_{\text{prec}}^{@p}$, then $\mathcal{A}$ is also a \textit{realizable} PAC learner for $\mathcal{H}$ w.r.t $\ell$.  Next, we will show how to convert the realizable PAC learner w.r.t $\ell$ into an agnostic PAC learner w.r.t $\ell$ in a black-box fashion. The composition of these two pieces yields an agnostic PAC learner for $\mathcal{H}$ w.r.t $\ell$.

If $\mathcal{H}$ is agnostic PAC learnable w.r.t $\ell_{\text{prec}}^{@p}$, then there exists a learning algorithm $\mathcal{A}$ with sample complexity $m(\epsilon, \delta, K)$ s.t. for any distribution $\mathcal{D}$ over $\mathcal{X} \times \mathcal{Y}$, with probability $1-\delta$ over a sample $S \sim \mathcal{D}^{n}$ of size $n \geq m(\epsilon, \delta, K)$, the output $g = \mathcal{A}(S)$ achieves

$$\mathbbm{E}_{\mathcal{D}}\left[\ell_{\text{prec}}^{@p}(g(x), y)) \right] \leq \inf_{h \in \mathcal{H}}\mathbbm{E}_{\mathcal{D}}\left[\ell_{\text{prec}}^{@p}(h(x), y)) \right] + \epsilon.$$

If $\mathcal{D}$ is realizable w.r.t $\ell$, then we are guaranteed that there exists a hypothesis $h^{\star} \in \mathcal{H}$ s.t.  $\mathbbm{E}_{\mathcal{D}}\left[\ell(h^{\star} (x), y) \right] = 0$. Since $\ell \in \mathcal{L}(\ell_{\text{prec}}^{@p})$, this also means that $\mathbbm{E}_{\mathcal{D}}\left[\ell_{\text{prec}}^{@p}(h^{\star}(x), y) \right] = 0$.  Furthermore, since $\ell \in \mathcal{L}(\ell_{\text{prec}}^{@p})$,   $\ell \leq  b\ell_{\text{prec}}^{@p}$. Together, this means we have   $\mathbbm{E}_{\mathcal{D}}\left[\ell(g (x), y) \right] \leq  b\epsilon$ showing have that $\mathcal{A}$ is also a realizable PAC learner for $\mathcal{H}$ w.r.t $\ell$ with sample complexity $m(\frac{\epsilon}{b}, \delta, K)$. This completes the first part of the proof. 

Now, we show how to convert the realizable PAC learner $\mathcal{A}$ for $\ell$ into an agnostic PAC learner for $\ell$ in a black-box fashion. For this step, we will use a similar algorithm as in the proof of Lemma \ref{lem:batch_sumloss2arbloss}. That is, we will show that Algorithm \ref{alg:batch_prec2arbell} below is an agnostic PAC learner for $\mathcal{H}$ w.r.t $\ell$.

\begin{algorithm}
\caption{Agnostic  PAC learner  for $\Hcal$ w.r.t. $\ell$}
\label{alg:batch_prec2arbell}
\setcounter{AlgoLine}{0}
\KwIn{Realizable PAC  learner $\Acal$ for $\Hcal$ w.r.t $\ell$, unlabeled samples $S_U \sim \Dcal_{\Xcal}^n$, and labeled samples $S_L \sim \Dcal^m$ }

For each $h \in \mathcal{H}_{|S_U}$, construct a dataset
$$S_U^h = \{(x_1, \tilde{y}_1), ..., (x_n,\tilde{y}_n)\} \text{ s.t. } \tilde{y}_i = \text{BinRel}(h(x_i), p)$$

Run $\Acal$ over all datasets to get $C(S_U) := \left\{\Acal\big(S_U^h \big) 
 \mid  h \in \mathcal{H}_{|S_U}\right\}$

Return $\hat{g} \in C(S_U)$  with the lowest empirical error over $S_L$ w.r.t. $\ell$. 
\end{algorithm}

 Let $\mathcal{D}$ be any (not necessarily realizable) distribution over $\mathcal{X} \times \mathcal{Y}$. Let $h^{\star} = \argmin_{h \in \mathcal{H}} \mathbbm{E}_{\mathcal{D}}\left[\ell(h(x), y)) \right]$ denote the optimal predictor in $\mathcal{H}$ w.r.t $\mathcal{D}$. Consider the sample $S_U^{h^{\star}}$ and let $g = \mathcal{A}(S_U^{h^{\star}})$. We can think of $g$ as the output of $\mathcal{A}$ run over an i.i.d sample $S$ drawn from $\mathcal{D}^{\star}$,  a joint distribution over $\mathcal{X} \times \mathcal{Y}$ defined procedurally by first sampling $x \sim \mathcal{D}_{\mathcal{X}}$, and then outputting the labeled sample $(x, \text{BinRel}(h^{\star}(x), p))$. 
 Note that $\mathcal{D}^{\star}$ is indeed a realizable distribution (realized by $h^{\star}$) w.r.t both $\ell$ and $\ell_{\text{prec}}^{@p}$. Recall that $m_{\Acal}(\frac{\epsilon}{b}, \delta, K)$ is the sample complexity of $\mathcal{A}$. Since $\mathcal{A}$ is a realizable learner for $\mathcal{H}$ w.r.t $\ell$, we have that for $n \geq m_{\Acal}(\frac{a\epsilon}{2b^2}, \delta/2, K)$, with probability at least $1-\frac{\delta}{2}$,

$$\mathbbm{E}_{\mathcal{D}^{\star}}\left[\ell(g(x), y) \right] \leq \frac{a\epsilon}{2b}.$$

 By definition of $\mathcal{D}^{\star}$, it further follows that $\mathbbm{E}_{\mathcal{D}^{\star}}\left[\ell(g(x), y) \right] = \mathbbm{E}_{x \sim \mathcal{D}_{\mathcal{X}}}\left[\ell(g(x), \text{BinRel}(h^{\star}(x), p)) \right]$. Therefore, 

$$\mathbbm{E}_{x \sim \mathcal{D}_{\mathcal{X}}}\left[\ell(g(x), \text{BinRel}(h^{\star}(x), p)) \right] \leq \frac{a\epsilon}{2b}.$$

Next, by Lemma \ref{lem:subaddprec@p}, we have pointwise that: 

$$\ell(g(x), y) \leq \ell(h^{\star}(x), y) + \frac{b}{a}\ell(g(x), \text{BinRel}(h^{\star}(x), p)).$$




Taking expectations on both sides of the inequality gives: 

\begin{align*}
\mathbbm{E}_{\mathcal{D}}\left[ \ell(g(x), y) \right] &\leq \mathbbm{E}_{\mathcal{D}}\left[\ell(h^{\star}(x), y) \right] + \mathbbm{E}_{\mathcal{D}}\left[\frac{b}{a}\ell(g(x), \text{BinRel}(h^{\star}(x), p)) \right]\\
&= \mathbbm{E}_{\mathcal{D}}\left[\ell(h^{\star}(x), y) \right] + \frac{b}{a}\mathbbm{E}_{x \sim \mathcal{D}_{\mathcal{X}}}\left[\ell(g(x), \text{BinRel}(h^{\star}(x), p)) \right]\\
&\leq \mathbbm{E}_{\mathcal{D}}\left[\ell(h^{\star}(x), y) \right] + \frac{\epsilon}{2}.
\end{align*}

Therefore, we have shown that $C(S_U)$ contains a hypothesis $g$ that generalizes well with respect to $\mathcal{D}$. The remaining proof follows exactly as in the proof of Lemma \ref{lem:batch_sumloss2arbloss}. We include them here for the sake of completeness. 

Now we want to show that the predictor $\hat{g}$ returned in step 4 also has good generalization. Crucially, observe that $C(S_U)$ is a finite hypothesis class with cardinality at most  $K^{pn}$. Therefore, by standard Chernoff and union bounds, with probability at least $1-\delta/2$, the empirical risk of every hypothesis in $C(S_U)$ on a sample of size $\geq  \frac{8}{\epsilon^2} \log{\frac{4 |C(S_U)|}{\delta}}$ is at most $\epsilon/4$ away from its true error. So, if $m = |S_L| \geq  \frac{8}{\epsilon^2} \log{\frac{4 |C(S_U)|}{\delta}}$, then with probability at least $1-\delta/2$, we have
\[\frac{1}{|S_L|} \sum_{(x, y) \in S_L} \ell(g(x), y)  \leq \expect_{\Dcal}\left[\ell(g(x), y)\right] + \frac{\epsilon}{4} \leq \mathbbm{E}_{\mathcal{D}}\left[\ell(h^{\star}(x), y) \right] + \frac{3\epsilon}{4}. \]
Since $\hat{g}$ is the ERM on $S_L$ over $C(S)$, its empirical risk can be at most $\mathbbm{E}_{\mathcal{D}}\left[\ell(h^{\star}(x), y) \right] + \frac{3\epsilon}{4}$. Given that the population risk of $\hat{g}$ can be at most $\epsilon/4$ away from its empirical risk, we have that
\[\expect_{\Dcal}[\ell(\hat{g}(x), y)]  \leq \mathbbm{E}_{\mathcal{D}}\left[\ell(h^{\star}(x), y) \right] + \epsilon. \]
Applying union bounds, the entire process succeeds with probability $1- \delta$.  We can upper bound the sample complexity of Algorithm \ref{alg:batch_sum2arbell}, denoted $n(\epsilon, 
\delta, K)$, as 
\begin{equation*}
    \begin{split}
       n(\epsilon, \delta, K) &\leq m_{\mathcal{A}}(\frac{a\epsilon}{2b^2}, \delta/2, K) + O \left( \frac{1}{\epsilon^2} \log{\frac{|C(S_U)|}{\delta}}\right)\\
       &\leq  m_{\mathcal{A}}(\frac{a\epsilon}{2b^2}, \delta/2, K) + O \left( \frac{ p \,m_{\mathcal{A}}(\frac{a\epsilon}{2b^2}, \delta/2, K) \log(K)\,+ \log{\frac{1}{\delta}}}{\epsilon^2} \right),
    \end{split}
\end{equation*}
where we use $|C(S_U)| \leq K^{pm_{\mathcal{A}}(\frac{ a\epsilon}{2b^2}, \delta/2, K)}$. This shows that Algorithm \ref{alg:batch_sum2arbell}, given as input an realizable PAC learner for $\mathcal{H}$ w.r.t $\ell$, is an agnostic PAC learner for $\mathcal{H}$ w.r.t $\ell$. Using the realizable learner we constructed before this step as the input completes this proof as we have constructively converted an agnostic PAC learner for $\ell_{\text{prec}}^{@p}$ into an agnostic PAC learner for $\ell$.
\end{proof}

Lemma \ref{lem:batch_suffprecloss} and \ref{lem:batch_precloss2arbloss} together complete the proof of sufficiency in Theorem \ref{thm:batch_prec@p}. Finally, Lemma \ref{lem:batch_prec@pnec} below shows that the agnostic PAC learnability of $\mathcal{H}_i^p$ for all $i \in [K]$ is necessary for the agnostic PAC learnability of $\mathcal{H}$ w.r.t any $\ell \in \mathcal{L}(\ell_{\text{prec}}^{@p})$. Like before, the proof 
of Lemma \ref{lem:batch_prec@pnec} is constructive and follows exactly the same strategy as Lemma \ref{lem:batch_sum@pnec}.  That is, given as input a learner for $\ell$, we will convert it into an agnostic learner for $\mathcal{H}_i^p$. In fact, the conversion is exactly the same as in the proof of Lemma \ref{lem:batch_sum@pnec} and just requires running Algorithm \ref{alg:batch_realHij} with an input learner for $\ell \in \mathcal{L}(\ell_{\text{prec}}^{@p})$ and setting $j = p$. 

\begin{lemma} \label{lem:batch_prec@pnec}
 If a function class $\mathcal{H} \subset \mathcal{S}_K^{\mathcal{X}}$ is agnostic PAC learnable w.r.t $\ell \in \mathcal{L}(\ell^{@p}_{\text{prec}})$, then $\mathcal{H}_i^p$ is agnostic PAC learnable w.r.t the 0-1 loss for all $i \in [K]$.
\end{lemma}

\begin{proof}
Fix $\ell \in \mathcal{L}(\ell_{\text{prec}}^{@p})$ and $i \in [K]$. Let $a = \min_{\pi, y}\{ \ell(\pi, y) \, \mid \, \ell(\pi, y) \neq 0\}$. Let $\mathcal{H}$ be an arbitrary ranking hypothesis class and $\mathcal{A}$ be an agnostic PAC learner for $\mathcal{H}$ w.r.t $\ell$. Our goal will to be to use $\mathcal{A}$ to construct an agnostic PAC learner for $\mathcal{H}_i^p$.
    
 Let $\mathcal{D}$ be any distribution over $\mathcal{X} \times \{0, 1\}$, $h^{ \star, p}_i = \argmin_{h \in \mathcal{H}_i^p}\mathbbm{E}_{\mathcal{D}}\left[\mathbbm{1}\{h(x) \neq y\} \right]$ the optimal hypothesis, and $h^{\star} \in \mathcal{H}$ be any valid completion of $h^{ \star, p}_i$. We will now show that Algorithm \ref{alg:batch_realHij} from the proof of Lemma \ref{lem:batch_sum@pnec} is an agnostic PAC learner for $\mathcal{H}_i^p$ if we set $j = p$ and give it as input an agnostic PAC learner $\mathcal{A}$ for $\mathcal{H}$ w.r.t. $\ell \in \mathcal{L}(\ell_{\text{prec}}^{@p})$.

Consider the sample $S_U^{h^{\star}}$ and let $g = \mathcal{A}(S_U^{h^{\star}})$. We can think of $g$ as the output of $\mathcal{A}$ run over an i.i.d sample $S$ drawn from $\mathcal{D}^{\star}$,  a joint distribution over $\mathcal{X} \times \mathcal{Y}$ defined procedurally by first sampling $x \sim \mathcal{D}_{\mathcal{X}}$ and then outputting the labeled sample $(x, \text{BinRel}(h^{\star}(x), p))$. 
Note that $\mathcal{D}^{\star}$ is a realizable distribution (realized by $h^{\star}$) w.r.t $\ell_{\text{prec}}^{@p}$ and therefore also $\ell$. Let $m_{\Acal}(\epsilon, \delta, K)$ be the sample complexity of $\mathcal{A}$. 

Since $\mathcal{A}$ is an agnostic PAC learner for $\mathcal{H}$ w.r.t $\ell$, we have that for sample size $n \geq m_{\Acal}(\frac{a\epsilon}{2}, \delta/2, K)$, with probability at least $1-\frac{\delta}{2}$,

$$\mathbbm{E}_{\mathcal{D}^{\star}}\left[\ell(g(x), y) \right] \leq \inf_{h \in \mathcal{H}}\mathbbm{E}_{\mathcal{D}^{\star}}\left[\ell(h(x), y) \right] + \frac{a\epsilon}{2} = \frac{a\epsilon}{2}.$$

Furthermore, by definition of $\mathcal{D}^{\star}$, $\mathbbm{E}_{\mathcal{D}^{\star}}\left[\ell(g(x), y) \right] = \mathbbm{E}_{x \sim \mathcal{D}_{\mathcal{X}}}\left[\ell(g(x), \text{BinRel}(h^{\star}(x), p)) \right]$. Therefore, $\mathbbm{E}_{x \sim \mathcal{D}_{\mathcal{X}}}\left[\ell(g(x), \text{BinRel}(h^{\star}(x), p)) \right] \leq \frac{a\epsilon}{2}.$  Next, using Lemma \ref{lem:prec@plb}, we have pointwise that

\begin{equation*}
    \begin{split}
        \indicator\{g_i^p(x) \neq h_i^{ \star, p}(x)\} &\leq \indicator\{\ell_{\text{prec}}^{@p}(g(x), \text{BinRel}(h^{\star}(x), p)) > 0\} \\
        &=  \indicator\{\ell(g(x), \text{BinRel}(h^{\star}(x), p)) > 0\} \\
        &\leq \frac{1}{a} \, \ell(g(x), \text{BinRel}(h^{\star}(x), p)).
    \end{split}
\end{equation*}

Taking expectations on both sides gives, 
$$
    \mathbbm{E}_{\mathcal{D}}\left[  \indicator\{g_i^p(x) \neq h_i^{ \star, p}(x)\}\right] \leq \frac{1}{a} \,\mathbbm{E}_{\mathcal{D}}\left[ \ell(g(x), \text{BinRel}(h^{\star}(x), p)) \right] 
    \leq \frac{\epsilon}{2},
$$
where in the last inequality we use the fact that $\mathbbm{E}_{x \sim \mathcal{D}_{\mathcal{X}}}\left[\ell(g(x), \text{BinRel}(h^{\star}(x), p)) \right] \leq \frac{a\epsilon}{2}.$ Finally, using the triangle inequality, we have that 

\begin{align*}
    \mathbbm{E}_{\mathcal{D}}\left[  \indicator\{g_i^p(x) \neq y\}\right] &\leq \mathbbm{E}_{\mathcal{D}}\left[  \indicator\{h^{\star, p}_i(x) \neq y\}\right] + \mathbbm{E}_{\mathcal{D}}\left[  \indicator\{g_i^p(x) \neq h^{\star, p}_i(x\}\right] \\
    &\leq \mathbbm{E}_{\mathcal{D}}\left[  \indicator\{h^{\star, p}_i(x) \neq y\}\right] + \frac{\epsilon}{2}\\
    &= \argmin_{h_i^p \in \mathcal{H}_i^p}\mathbbm{E}_{\mathcal{D}}\left[\mathbbm{1}\{h_i^p(x) \neq y\} \right] + \frac{\epsilon}{2} .
\end{align*}

Since $g_i^p \in C_i^p(S_U)$,  we have shown that $C_i^p(S_U)$ contains a hypothesis that generalizes well w.r.t $\mathcal{D}$. Now we want to show that the predictor $\hat{g}_i^p$ returned in step 4 also generalizes well. Crucially, observe that $C_i^p(S_U)$ is a finite hypothesis class with cardinality at most  $K^{pn}$. Therefore, by standard Chernoff and union bounds, with probability at least $1-\delta/2$, the empirical risk of every hypothesis in $C_i^p(S_U)$ on a sample of size $\geq  \frac{8}{\epsilon^2} \log{\frac{4 |C_i^j(S_U)|}{\delta}}$ is at most $\epsilon/4$ away from its true error. So, if $m = |S_L| \geq  \frac{8}{\epsilon^2} \log{\frac{4 |C_i^j(S_U)|}{\delta}}$, then with probability at least $1-\delta/2$, we have
\[\frac{1}{|S_L|} \sum_{(x, y) \in S_L} \mathbbm{1}\{g_i^p(x) \neq  y  \}\leq \expect_{\Dcal}\left[\mathbbm{1}\{g_i^p(x) \neq  y  \}\right] + \frac{\epsilon}{4} \leq \frac{3\epsilon}{4}. \]
Since $\hat{g}_i^p$ is the ERM on $S_L$ over $C_i^p(S_U)$, its empirical risk can be at most $\frac{3\epsilon}{4}$. Given that the population risk of $\hat{g}_i^p$ can be at most $\epsilon/4$ away from its empirical risk, we have that
\[\expect_{\Dcal}[\mathbbm{1}\{\hat{g}_i^p(x) \neq y\}]  \leq \argmin_{h_i^p \in \mathcal{H}_i^p}\mathbbm{E}_{\mathcal{D}}\left[\mathbbm{1}\{h_i^p(x) \neq y\} \right] + \epsilon. \]

Applying union bounds, the entire process succeeds with probability $1- \delta$.  We can compute the upper bound on the sample complexity of Algorithm \ref{alg:batch_realHij}, denoted $n(\epsilon, 
\delta, K)$, as 
\begin{equation*}
    \begin{split}
       n(\epsilon, \delta, K) &\leq m_{\mathcal{A}}(\frac{a\epsilon}{2}, \delta/2, K) + O \left( \frac{1}{\epsilon^2} \log{\frac{|C(S_U)|}{\delta}}\right)\\
       &\leq  m_{\mathcal{A}}(\frac{a\epsilon}{2}, \delta/2, K) + O \left( \frac{ p \; m_{\mathcal{A}}(\frac{a\epsilon}{2}, \delta/2, K)\log(K)\,+ \log{\frac{1}{\delta}}}{\epsilon^2} \right),
    \end{split}
\end{equation*}
where we use $|C(S_U)| \leq K^{p m_{\mathcal{A}}(\frac{a\epsilon}{2}, \delta/2, K)}$. This shows that Algorithm \ref{alg:batch_realHij} is an agnostic PAC learner for $\mathcal{H}_i^p$ w.r.t 0-1 loss.  Since our choice of loss $\ell \in \mathcal{L}(\ell_{\text{prec}}^{@p})$ and index $i$ were arbitrary, agnostic PAC learnability of $\mathcal{H}$ w.r.t $\ell$ implies agnostic PAC learnability of $\mathcal{H}_i^p$ w.r.t the 0-1 loss for all $i \in [K]$. \end{proof}

Combining Lemma \ref{lem:batch_suffprecloss}, \ref{lem:batch_precloss2arbloss} and \ref{lem:batch_prec@pnec} gives Theorem \ref{thm:batch_prec@p}. 

\section{Proofs for Online Multilabel Ranking}

\subsection{Proof of necessity in Theorem \ref{thm:ol_sum@p}}\label{appdx:necc_thm:ol_sum@p}

\begin{proof}
   Fix $\ell \in \Lcal(\ell_{\text{sum}}^{@p})$ and $(i, j) \in [K] \times [p]$. Given an online learner $\Acal$ for $\Hcal$ w.r.t $\ell$, our goal is to  construct an agnostic online learner $\Acal_i^j$ for $\Hcal_i^j$. To that end, let $(x_1, y_1), ..., (x_T,  y_T) \in (\mathcal{X} \times \{0,1\})^T$ denote a stream of labeled instances. Define $h_{i}^{\star, j} = \argmin_{h_i^{ j} \in \Hcal_{i}^j} \sum_{t=1}^T \indicator\{h_{i}^j(x_t) \neq y_t\}$ to be the optimal function in $\Hcal_{i}^j$ and $h^{\star}$ be an arbitrary completion of $h_{i}^{\star, j}$.  As in the sufficiency proof, our construction of the online learner for $\Hcal_i^j$ will run REWA over a set of experts  we construct below. 

  For any bitstring $b \in \{0, 1\}^T$, let $\phi: \{t \in [T]: b_t = 1\} \rightarrow \Scal_{K}$ denote a function mapping time points where $b_t = 1$ to permutations. Let $\Phi_b = \Scal_K^{\{t \in [T]: b_t = 1\}}$ denote all such functions $\phi$. For every $h\in \Hcal $, there exists a $\phi_b^{h} \in \Phi_b$  such that for all $t \in \{t: b_t = 1\}$, $\phi_b^h(t) = h(x_t)$. Let $|b| = |\{t \in [T]: b_t = 1\}|$. For every $b \in \{0, 1\}^T$ and $\phi \in \Phi_b$, define an Expert $E_{b, \phi}$. Expert $E_{b, \phi}$, formally presented in Algorithm \ref{alg:ness_expert_sum@p}, uses $\mathcal{A}$ to make predictions in each round.  For every $b \in \{0, 1\}^T$, let $\mathcal{E}_b = \bigcup_{\phi \in \Phi_b} \{E_{b, \phi}\}$ denote the set of all Experts parameterized by functions $\phi \in \Phi_b$. As before, we will actually define $\mathcal{E}_b = \{E_0\} \cup \bigcup_{\phi \in \Phi_b} \{E_{b, \phi}\}$, where $E_0$ is the expert that never updates $\mathcal{A}$ and only uses it to make predictions in each round.  Note that $1 \leq |\mathcal{E}_b| \leq (K!)^{|b|} \leq K^{K |b|}$.

\begin{algorithm}
\caption{Expert $(b, \phi)$}
\label{alg:ness_expert_sum@p}
\setcounter{AlgoLine}{0}
\KwIn{Independent copy of online learner $\Acal$ for $\Hcal$}
\For{$t = 1,...,T$} {
    Receive example $x_t$
    
    Predict $\indicator\{\hat{\pi}_i \leq j\}$ where $\hat{\pi} = \Acal(x_t)$

      \uIf{$b_t = 1$}{        
            Update $\mathcal{A}$ by passing $(x_t, \text{BinRel}(\phi(t), j))$
        }
}
\end{algorithm}

We are now ready to give the agnostic online learner for $\Hcal_i^j$, henceforth denoted by $\Qcal$. Our online learner $\mathcal{Q}$ is very similar to Algorithm \ref{alg:agn_lsum@p}. First, it will sample a $B \in \{0, 1\}^T$ s.t. $B_t \sim \text{Bernoulli}(T^{\beta}/T)$. Then, it will construct a set of experts $\Ecal_B$ using Algorithm \ref{alg:ness_expert_sum@p}. Finally, it will run REWA, denoted by $\Pcal$, on the 0-1 loss over the stream $(x_1, y_1), ..., (x_T,  y_T)$.  As before, let $A$ and $P$ be the random variables denoting  internal randomness of the algorithm $\Acal$ and $\Pcal$.  Using REWA guarantees and following exactly the same calculation as in the sufficiency proof, we arrive at
\[ \mathbbm{E}\left[\sum_{t=1}^T \indicator\{\mathcal{Q}(x_t) \neq y_t\} \right] \leq \mathbbm{E}\left[\sum_{t=1}^T \indicator\{E_{B, \phi_{B}^{h^{\star}}}(x_t) \neq y_t\}  \right] + \sqrt{2 T^{1+\beta} K \ln{ K}}. \]
The inequality above is the adaptation of Equation \eqref{REWA_guarantee} for this proof. Recall that $h_{i}^{\star, j}$ is the optimal function in hindsight for the stream and $h^{\star}$ is a completion of $h_{i}^{\star,j}$. Since $\indicator\{E_{B, \phi_{B}^{h^{\star}}}(x_t) \neq y_t\} \leq  \indicator\{h_i^{\star, j}(x_t) \neq y_t\} \, +\,  \indicator\{E_{B, \phi_{B}^{h^{\star}}}(x_t) \neq h_i^{\star, j}(x_t)  \}$, the inequality above reduces to
\begin{equation*}
    \begin{split}
        \mathbbm{E}\left[\sum_{t=1}^T \indicator\{\mathcal{Q}(x_t) \neq y_t\} \right] \leq \sum_{t=1}^T \indicator\{h_i^{\star, j}(x_t) \neq y_t\}   +  \mathbbm{E}\left[\sum_{t=1}^T \indicator\{E_{B, \phi_{B}^{h^{\star}}}(x_t) \neq h_i^{\star, j}(x_t)  \}  \right] + \sqrt{2 T^{1+\beta} K \ln{ K}}.
    \end{split}
\end{equation*}

It now suffices to show that $\mathbbm{E}\left[\sum_{t=1}^T \indicator\{E_{B, \phi_{B}^{h^{\star}}}(x_t) \neq h_i^{\star, j}(x_t) \}  \right]$ is sub-linear function of $T$.

Given an online learner $\mathcal{A}$ for $\Hcal$, an instance $x \in \mathcal{X}$, and an ordered finite sequence of labeled examples $L \in (\mathcal{X} \times \mathcal{Y})^*$, let $\mathcal{A}(x|L)$ be the random variable denoting the prediction of $\mathcal{A}$ on the instance $x$ after running and updating on $L$. For any $b\in \{0, 1\}^T$, $h \in \mathcal{H}$, and $t \in [T]$, let $L^h_{b_{< t}} = \{(x_i, \text{BinRel}(h(x_s), j)): s < t \text{ and } b_s = 1\}$ denote the \textit{subsequence} of the sequence of labeled instances $\{(x_s, \text{BinRel}(h(x_s), j))\}_{s=1}^{t-1}$ where $b_s = 1$. Thus, using Lemma \ref{lem:sum@plb}, we have
\begin{equation*}
    \begin{split}
        \indicator\{E_{B, \phi_{B}^{h^{\star}}}(x_t) \neq h_i^{ \star, j}(x_t)\} &\leq \indicator\{\ell_{\text{sum}}^{@p}(\Acal(x_t \mid L^{h^{\star}}_{B_{< t}}), \text{BinRel}(h^{\star}(x_t), j)) > 0\} \\
        &=  \indicator\{\ell(\Acal(x_t \mid L^{h^{\star}}_{B_{< t}}), \text{BinRel}(h^{\star}(x_t), j)) > 0\} \\
        &\leq \frac{1}{a} \, \ell(\Acal(x_t \mid L^{h^{\star}}_{B_{< t}}), \text{BinRel}(h^{\star}(x_t), j), \text{BinRel}(h^{\star}(x_t), j)),
    \end{split}
\end{equation*}
where equality follows from the fact that $\ell \in \mathcal{L}(\ell_{\text{sum}}^{@p})$. Here, $a$ is the lower bound whenever it is non-zero. Taking expectations of both sides and summing over $t \in [T]$ gives

\[\mathbbm{E}\left[\sum_{t=1}^T \indicator\{E_{B, \phi_{B}^{h^{\star}}}(x_t) \neq h_i^{\star,j}(x_t)\}  \right] \leq \frac{1}{a}\mathbbm{E}\left[\sum_{t=1}^T  \ell(\Acal(x_t \mid L^{h^{\star}}_{B_{< t}}), \text{BinRel}(h^{\star}(x_t), j))  \right].  \]

To upperbound the right-hand side, we will again use the fact that the prediction $\Acal(x_t \mid L^{h^{\star}}_{B_{< t}})$ only depends on $(B_1, \ldots, B_{t-1})$, but is independent of $ B_t$. The details of this calculation are omitted because they are identical to that of the sufficiency proof. Using independence of $\Acal(x_t \mid L^{h^{\star}}_{B_{< t}})$ and $B_t$, we obtain
\begin{equation*}
    \begin{split}
    \mathbbm{E}\left[\sum_{t=1}^T  \ell(\Acal(x_t \mid L^{h^{\star}}_{B_{< t}}), \text{BinRel}(h^{\star}(x_t), j))  \right] &=  \frac{T}{T^{\beta}}\, \mathbbm{E}\left[\sum_{t : B_t =1} \ell(\Acal(x_t \mid L^{h^{\star}}_{B_{< t}}), \text{BinRel}(h^{\star}(x_t), j))  \right] \\
    &=  \frac{T}{T^{\beta}}\, \mathbbm{E} \left[\mathbbm{E}\left[\sum_{t : B_t =1} \ell(\Acal(x_t \mid L^{h^{\star}}_{B_{< t}}), \text{BinRel}(h^{\star}(x_t), j))  
 \, \Bigg| B \right] \right] \\
    &\leq \frac{T}{T^{\beta}} \expect \left[R(|B|, K) \right],
\end{split}
\end{equation*}
where $R(|B|, K)$ is the regret of the algorithm $\Acal$, a sub-linear function of $|B|$. In the last step, we use the fact that $\Acal$ is a (realizable) online learner for $\Hcal$ w.r.t. $\ell$ and the feedback that the algorithm received was $(x_t, \text{BinRel}(h^{\star}(x_t), j))$ in the rounds whenever $B_t =1$. Again, using Lemma 5.17 from \cite{woess2017groups} and Jensen's inequality yields $\mathbbm{E}\left[ R(|B|, K) \right] \leq  R(T^{\beta}, K)$, a sub-linear function of $T^{\beta}$. Combining everything, we get
\begin{equation*}
    \begin{split}
        \mathbbm{E}\left[\sum_{t=1}^T \indicator\{\mathcal{Q}(x_t) \neq y_t\} \right] &\leq \sum_{t=1}^T \indicator\{h_i^{\star, j}(x_t) \neq y_t\} +   \frac{T}{a T^{\beta}} R(T^{\beta}, K) + \sqrt{2 T^{1+\beta} K \ln{ K}} \\
        &= \argmin_{h_{i}^j \in \Hcal_{i}^j}\sum_{t=1}^T \indicator\{h_i^{ j}(x_t) \neq y_t\} +   \frac{T}{aT^{\beta}} R(T^{\beta}, K) + \sqrt{2 T^{1+\beta} K \ln{ K}}
        \end{split}
\end{equation*}

For any choice of $\beta \in (0, 1)$, the regret above is a sub-linear function of $T$. Therefore, we have shown that $\Qcal$ is an agnostic learner for $\Hcal_i^j$ w.r.t. $0$-$1$ loss. 
\end{proof}

\subsection{Proof of Theorem \ref{thm:ol_prec@p}}\label{appdx:proof_ol_prec@p}

\begin{proof}(of sufficiency in Theorem \ref{thm:ol_prec@p})
Fix $\ell \in \mathcal{L}(\ell_{\text{prec}}^{@p})$ and let $M = \max_{\pi, y}\ell(\pi, y) $. This proof is virtually identical to the proof of sufficiency in Theorem \ref{thm:batch_sum@p}. However, we provide the full details here for completion. Our proof is also based on reduction. That is, given realizable learners $\Acal_i^p$ of $\Hcal_i^p$'s for $i \in [K]$ w.r.t. $0$-$1$ loss, we will construct an agnostic learner $\Qcal$ for $\Hcal$ w.r.t. $\ell$. We will construct a set of experts $\Ecal$ that uses $\Acal_i^p$ to make predictions and run the REWA algorithm using these experts. 

Let $(x_1, y_1), ..., (x_T, y_T) \in (\mathcal{X} \times \mathcal{Y})^T$ denote the stream of points to be observed by the online learner. As before, we will assume an oblivious adversary. Define $h^{\star} = \argmin_{h \in \Hcal} \sum_{t=1}^T \ell(h(x_t), y_t)$ to be the optimal hypothesis in hindsight.

For any bitstring $b \in \{0, 1\}^T$, let $\phi: \{t \in [T]: b_t = 1\} \rightarrow \Scal_{K}$ denote a function mapping time points where $b_t = 1$ to permutations. Let $\Phi_b = \Scal_K^{\{t \in [T]: b_t = 1\}}$ denote all such functions $\phi$. For every $h \in \Hcal $, there exists a $\phi_b^h \in \Phi_b$  such that for all $t \in \{t: b_t = 1\}$, $\phi_b^h(t) = h(x_t)$. Let $|b| = |\{t \in [T]: b_t = 1\}|$. For every $b \in \{0, 1\}^T$ and $\phi \in \Phi_b$, we will define an Expert $E_{b, \phi}$. Expert $E_{b, \phi}$, formally presented in Algorithm \ref{alg:agn_lsum@p}, uses $\mathcal{A}_i^p$'s to make predictions in each round. However, $E_{b, \phi}$ only updates the $\mathcal{A}_i^p$'s on those rounds where $b_t = 1$, using $\phi$ to compute a labeled instance. For every $b \in \{0, 1\}^T$, let $\mathcal{E}_b = \bigcup_{\phi \in \Phi_b} \{E_{b, \phi}\}$ denote the set of all Experts parameterized by functions $\phi \in \Phi_b$. If $b$ is the bitstring with all zeros, then $\mathcal{E}_b$ will be empty. Therefore, we will actually define $\mathcal{E}_b = \{E_0\} \cup \bigcup_{\phi \in \Phi_b} \{E_{b, \phi}\}$, where $E_0$ is the expert that never updates $\mathcal{A}_i^j$'s and only uses them for predictions in all $t \in [T]$.  Note that $1 \leq |\mathcal{E}_b| \leq (K!)^{|b|} \leq K^{K |b|}$. Using these experts, Algorithm \ref{alg:agn_lsum@p} is our agnostic online learner $\mathcal{Q}$ for $\mathcal{\Hcal}$ w.r.t $\ell \in \Lcal(\ell_{\text{prec}}^{@ p})$.

\begin{algorithm}
\caption{Expert $(b, \phi)$}
\label{alg:expert_prec@p}
\setcounter{AlgoLine}{0}
\KwIn{Independent copy of realizable learners  $\Acal_i^p$ of $\Hcal_i^p$ for $i \in [K]$}
\For{$t = 1,...,T$} {
    Receive example $x_t$
    
    Define a binary vote vector $v_t \in \{0,1\}^{K}$ such that $v_t[i] = \Acal_i^p(x_t)$
    
    Predict $\hat{\pi}_t \in \argmin_{\pi \in \Scal_K} \langle \pi, v_t \rangle$

      \uIf{$b_t = 1$}{        
            Let $\pi = \phi(t)$ and for each $i \in [K]$, update  $\mathcal{A}_i^p$ by passing $(x_t, \pi_i^p)$
        }
}
\end{algorithm}

Using REWA guarantees and following exactly the same calculation as in the proof of Theorem \ref{thm:ol_sum@p} we immediately arrive at
\begin{equation*}
    \mathbbm{E}\left[\sum_{t=1}^T \ell(\mathcal{Q}(x_t), y_t) \right] \leq \mathbbm{E}\left[\sum_{t=1}^T \ell(E_{B, \phi_B^{h^{\star}}}(x_t), y_t) \right] + M \sqrt{2 T^{1+\beta} K \ln{ K}},
\end{equation*}
the analog of Equation \eqref{REWA_guarantee} for this setting. Using Lemma \ref{lem:subaddprec@p}, we have
\[\ell(E_{B, \phi_B^{h^{\star}}}(x_t), y_t) \leq \ell(h^{\star}(x_t), y_t) + \frac{M}{a} \ell(E_{B, \phi_B^{h^{\star}}}(x_t), \text{BinRel}(h^{\star}(x_t), p))\]
pointwise, where $a = \min_{\pi, y}\{ \ell(\pi, y) \, \mid \, \ell(\pi, y) \neq 0\}$. By definition of $M$, we further get 
\begin{equation*}
    \begin{split}
        \ell(E_{B, \phi_B^{h^{\star}}}(x_t), \text{BinRel}(h^{\star}(x_t), p)) &\leq M \,\indicator\{\ell(E_{B, \phi_B^{h^{\star}}}(x_t), \text{BinRel}(h^{\star}(x_t), p)) > 0\}  \\
        &= M\, \indicator\{\ell_{\text{prec}}^{@p}(E_{B, \phi_B^{h^{\star}}}(x_t), \text{BinRel}(h^{\star}(x_t), p)) > 0 \},
    \end{split}
\end{equation*}
where the equality follows from the fact that $\ell \in \Lcal(\ell_{\text{prec}}^{@p})$.

In order to upperbound the indicator above, we need some more notations. Given the realizable online learner $\mathcal{A}_i^p$ for $i \in [K] \times [p]$, an instance $x \in \mathcal{X}$, and an ordered finite sequence of labeled examples $L \in (\mathcal{X} \times \{0,1\})^*$, let $\mathcal{A}_i^p(x|L)$ be the random variable denoting the prediction of $\mathcal{A}_i^p$ on the instance $x$ after running and updating on $L$. For any $b\in \{0, 1\}^T$, $h \in \mathcal{H}$, and $t \in [T]$, let $L_{b_{< t}}^{h}(i,p) = \{(x_{s}, h_{i}^p(x_{s})): s < t \text{ and } b_{s} = 1\}$ denote the \textit{subsequence} of the sequence of labeled instances $\{(x_{s}, h_i^p(x_{s}))\}_{s=1}^{t-1}$ where $b_s = 1$. Then, we have
\begin{equation*}
    \begin{split}
       \indicator\{\ell_{\text{prec}}^{@p}(E_{B, \phi_B^{h^{\star}}}(x_t), \text{BinRel}(h^{\star}(x_t), p)) > 0 \} &\leq \sum_{i=1}^K  \indicator\{\Acal_i^{p}(x_t \mid L_{B_{< t}}^{h^{\star}}(i,p)) \neq h_i^{\star, p}(x_t)\}.
    \end{split}
\end{equation*}
To prove this claimed inequality, consider the case when  $\sum_{i=1}^K  \indicator\{\Acal_i^{p}(x_t \mid L_{B_{< t}}^{h^{\star}}(i,p)) \neq h_i^{\star, p}(x_t)\}  =0$ because the inequality is trivial otherwise. Then, we must have $ 
 \Acal_i^{p}(x_t \mid L_{B_{< t}}^{h^{\star}}(i,p)) = h_i^{ \star,p}(x_t)$ for all $i \in [K]$. Let $v_t \in \{0,1\}^{K }$ such that $v_t[i] =  \Acal_i^{p}(x_t \mid L_{B_{< t}}^{h^{\star}}(i,p))$ be a binary vote vector that the expert $E_{B, \phi_B^{h^{\star}}}$ constructs in round $t$. Since $h^{\star}(x_t)$ is a permutation, the vote vector $v_t$ must  contain exactly $p$ labels with $1$ vote and $K-p$ labels with $0$ votes. Thus, every $\hat{\pi}_t \in \argmin_{\pi \in \Scal_K} \langle \pi, v_t\rangle$ must rank labels with $1$ vote in top $p$ and labels with $0$ votes outside top $p$. In other words, we must have $\hat{\pi}_t \stackrel{\mathclap{p}}{=} h^{\star}(x_t)$, and thus $\ell_{\text{prec}}^{@p}(\hat{\pi}_t, \text{BinRel}(h^{\star}(x_t), p)) =0 $  by definition of $ \ell_{\text{prec}}^{@p}$. Our claim follows because $E_{B, \phi_B^{h^{\star}}}(x_t ) \in \argmin_{\pi \in \Scal_K} \langle \pi, v_t\rangle$.

Combining everything, we obtain
\[\ell(E_{B, \phi_B^{h^{\star}}}(x_t), y_t) \leq \ell(h^{\star}(x_t), y_t) +\frac{M^2}{a}\:\sum_{i=1}^K  \indicator\{\Acal_i^{p}(x_t \mid L_{B_{< t}}^{h^{\star}}(i,p)) \neq h_i^{\star,p}(x_t)\}. \]
Taking expectations on both sides and summing over all $t \in [T]$ yields
\[\mathbb{E}\left[\sum_{t=1}^T \ell(E_{B, \phi_B^{h^{\star}}}(x_t), y_t) \right] \leq \sum_{t=1}^T\ell(h^{\star}(x_t), y_t) + \frac{M^2}{a}\:  \sum_{i=1}^K  \expect \left[ \sum_{t=1}^T \indicator\{\Acal_i^{p}(x_t \mid L_{B_{< t}}^{h^{\star}}(i,p)) \neq h_i^{\star, p}(x_t)\} \right]. \]
So, it now suffices to show that $\expect \left[ \sum_{t=1}^T \indicator\{\Acal_i^{p}(x_t \mid L_{B_{< t}}^{h^{\star}}(i,p)) \neq h_i^{\star, p}(x_t)\} \right]$ is a sub-linear function of $T$. Again, using the independence of $B_t$ and the algorithm's prediction in round $t$, we can write
\begin{equation*}
    \begin{split}
        \expect \left[ \sum_{t=1}^T \indicator\{\Acal_i^{p}(x_t \mid L_{B_{< t}}^{h^{\star}}(i,p)) \neq h_i^{\star, p}(x_t)\}\right] &=    \sum_{t=1}^T \expect \left[ \indicator\{\Acal_i^{p}(x_t \mid L_{B_{< t}}^{h^{\star}}(i,p)) \neq h_i^{\star, p}(x_t)\} \right] \frac{\mathbb{P}\left[B_t = 1 \right]}{\mathbb{P}\left[B_t = 1 \right]}\\
        &= \frac{T}{T^{\beta}}\,  \sum_{t=1}^T \expect \left[\indicator\{\Acal_i^{p}(x_t \mid L_{B_{< t}}^{h^{\star}}(i,p)) \neq h_i^{\star, p}(x_t)\} \right] \mathbb{E}\left[ \indicator\{B_t = 1 \}\right]\\
        &= \frac{T}{T^{\beta}}\,  \sum_{t=1}^T \expect \left[\indicator\{\Acal_i^{p}(x_t \mid L_{B_{< t}}^{h^{\star}}(i,p)) \neq h_i^{\star, p}(x_t)\} \indicator\left\{B_t = 1 \right\}\right].
    \end{split}
\end{equation*}
 Next, we can use the regret guarantee of the algorithm $\Acal_i^{p}$ on the rounds it was updated. That is, 
 \begin{equation*}
    \begin{split}
     \sum_{t=1}^T \expect \left[\indicator\{\Acal_i^{p}(x_t \mid L_{B_{< t}}^{h^{\star}}(i,p)) \neq h_i^{\star, p}(x_t)\}\indicator\left\{B_t = 1 \right\}\right] &=     \expect \left[ \sum_{t  :  B_t =1} \indicator\{\Acal_i^{p}(x_t \mid L_{B_{< t}}^{h^{\star}}(i,p)) \neq h_i^{\star, p}(x_t)\}\right] \\
     &= \expect \left[ \expect \left[ \sum_{t  :  B_t =1} \indicator\{\Acal_i^{p}(x_t \mid L_{B_{< t}}^{h^{\star}}(i,p)) \neq h_i^{\star, p}(x_t)\} \, \Bigg| B \right]  \right]\\
     &\leq \expect_{B}\left[R_i^p(|B|) \right] ,
    \end{split}
\end{equation*}
where $R_i^p(|B|)$ is the regret of  $\Acal_i^{p}$, a sub-linear function of $|B|$. In the last step, we use the fact that $\Acal_i^{p}$ is a realizable algorithm for $\Hcal_i^p$ and the feedback that the algorithm received was $(x_t, h_i^{\star,p}(x_t))$ in the rounds whenever $B_t =1$. Without loss of generality, we assume that $R_i^p(|B|)$ is a concave function of $|B|$. Otherwise, by Lemma 5.17 from \cite{woess2017groups}, there exists a concave sub-linear function $\tilde{R}_i^p(|B|)$ that upperbounds $R_i^p(|B|)$. By Jensen's inequality, $\mathbbm{E}_B\left[ R_i^p(|B|) \right] \leq  R_i^p(T^{\beta})$, a sub-linear function of $T^{\beta}$. 

Putting everything together, we obtain
\begin{equation*}
    \begin{split}
        \mathbbm{E}\left[\sum_{t=1}^T \ell(\mathcal{Q}(x_t), y_t) \right] &\leq \sum_{t=1}^T \ell(h^{\star}(x_t), y_t) + \frac{M^2}{a} \sum_{i=1}^K  \frac{T}{T^{\beta}}\, R_i^{p}(T^{\beta})+ M \sqrt{2 T^{1+\beta} K \ln{ K}}\\
        &=  \inf_{h \in \Hcal} \sum_{t=1}^T \ell(h(x_t), y_t) + \frac{pM^2}{a} \sum_{i=1}^K \frac{T}{T^{\beta}}\, R_i^{p}(T^{\beta})+ M \sqrt{2 T^{1+\beta} K \ln{ K}}.
    \end{split}
\end{equation*}

 Since $R_i^{p}(T^{\beta})$ is a sublinear function of $T^{\beta}$, $\frac{T}{T^{\beta}}R_i^{p}(T^{\beta}) $ is a sublinear function of $T$. As the sum of sublinear functions is sublinear, the second term above must be a sublinear function of $T$. Thus, the regret is sub-linear for any choice of $\beta \in (0,1)$.  This completes our proof as we have shown that the algorithm $\Qcal$ achieves sub-linear regret in $T$. \end{proof} 

  We will now show that the online learnability of $\Hcal$ w.r.t $\ell$ implies that  $\Hcal_i^{p}$ for each $i \in [K]$ is online learnable w.r.t 0-1 loss.

\begin{proof}(of necessity in Theorem \ref{thm:ol_prec@p})

   Fix $\ell \in \mathcal{L}(\ell_{\text{prec}}^{@p})$ and let $M = \max_{\pi, y}\ell(\pi, y)$. Given an online learner $\Acal$ for $\Hcal$ w.r.t $\ell$, our goal is to construct an agnostic online learner $\Acal_i^p$ for $\Hcal_i^p$ for a fixed $i \in [K]$. One can construct agnostic online learners for $\Hcal_i^p$ for all $i \in [K]$ by symmetry. Our construction uses the REWA and is similar to the sufficiency proof above. 

 Let us define function $\phi$'s, the collection of functions $\Phi_b$ for every $b$ in the same way we did before. For every $b \in \{0, 1\}^T$ and $\phi \in \Phi_b$, define an Expert $E_{b, \phi}$. Expert $E_{b, \phi}$ is the expert presented in Algorithm \ref{alg:ness_expert_sum@p} after setting $j=p$ and uses $\mathcal{A}$ to make predictions in each round.  For every $b \in \{0, 1\}^T$, let $\mathcal{E}_b = \bigcup_{\phi \in \Phi_b} \{E_{b, \phi}\}$ denote the set of all Experts parameterized by functions $\phi \in \Phi_b$. As before, we will actually define $\mathcal{E}_b = \{E_0\} \cup \bigcup_{\phi \in \Phi_b} \{E_{b, \phi}\}$, where $E_0$ is the expert that never updates $\mathcal{A}$ and only uses it to make predictions in each round.  Note that $1 \leq |\mathcal{E}_b| \leq (K!)^{|b|} \leq K^{K |b|}$. 

    The online learner for $\Hcal_i^p$, henceforth denoted by $\Qcal$, is similar to Algorithm \ref{alg:agn_lsum@p}. First, it samples a $B \in \{0, 1\}^T$ s.t. $B_t \sim \text{Bernoulli}(T^{\beta}/T)$, constructs a set of experts $\Ecal_B$ using Algorithm \ref{alg:ness_expert_sum@p} and runs REWA, denoted by $\Pcal$, on the 0-1 loss over the stream $(x_1, y_1), ..., (x_T,  y_T) \in (\mathcal{X} \times \{0,1\})^T$. Let $h_i^{\star, p} = \argmin_{h_i^{p} \in \Hcal_i^p} \sum_{t=1}^T \indicator\{h_i^p(x_t) \neq y_t\}$ be the optimal function in hindsight and $h^{\star}$ be any arbitrary completion of $h_i^{\star, p}$.

  Using REWA guarantees and following exactly the same calculation as in the sufficiency proof, we arrive at
\[ \mathbbm{E}\left[\sum_{t=1}^T \indicator\{\mathcal{Q}(x_t) \neq y_t\} \right] \leq \mathbbm{E}\left[\sum_{t=1}^T \indicator\{E_{B, \phi_{B}^{h^{\star}}}(x_t) \neq y_t\}  \right] + \sqrt{2 T^{1+\beta} K \ln{ K}}. \]
The inequality above is the adaptation of Equation \eqref{REWA_guarantee} for this proof. Since $\indicator\{E_{B, \phi_{B}^{h^{\star}}}(x_t) \neq y_t\} \leq  \indicator\{h_i^{\star, p}(x_t) \neq y_t\} \, +\,  \indicator\{E_{B, \phi_{B}^{h^{\star}}}(x_t) \neq h_i^{\star, p}(x_t)  \}$, the inequality above reduces to
\begin{equation*}
    \begin{split}
        \mathbbm{E}\left[\sum_{t=1}^T \indicator\{\mathcal{Q}(x_t) \neq y_t\} \right] \leq \sum_{t=1}^T \indicator\{h_i^{\star, p}(x_t) \neq y_t\}   +  \mathbbm{E}\left[\sum_{t=1}^T \indicator\{E_{B, \phi_{B}^{h^{\star}}}(x_t) \neq h_i^{\star, p}(x_t)  \}  \right] + \sqrt{2 T^{1+\beta} K \ln{ K}}.
    \end{split}
\end{equation*}

It now suffices to show that $\mathbbm{E}\left[\sum_{t=1}^T \indicator\{E_{B, \phi_{B}^{h^{\star}}}(x_t) \neq h_i^{\star, p}(x_t) \}  \right]$ is sub-linear in $T$.

Given an online learner $\mathcal{A}$ for $\Hcal$, an instance $x \in \mathcal{X}$, and an ordered finite sequence of labeled examples $L \in (\mathcal{X} \times \mathcal{Y})^*$, let $\mathcal{A}(x|L)$ be the random variable denoting the prediction of $\mathcal{A}$ on the instance $x$ after running and updating on $L$. For any $b\in \{0, 1\}^T$, $h \in \mathcal{H}$, and $t \in [T]$, let $L^h_{b_{< t}} = \{(x_i, \text{BinRel}(h(x_s), p)): s < t \text{ and } b_s = 1\}$ denote the \textit{subsequence} of the sequence of labeled instances $\{(x_s, \text{BinRel}(h(x_s), p))\}_{s=1}^{t-1}$ where $b_s = 1$.  Using Lemma \ref{lem:prec@plb}, we have 
\begin{equation*}
    \begin{split}
        \indicator\{E_{B, \phi_{B}^{h^{\star}}}(x_t) \neq h_i^{\star,p}(x_t)\} &\leq \indicator\{\ell_{\text{prec}}^{@p}(\Acal(x_t \mid L^{h^{\star}}_{B_{< t}}), \text{BinRel}(h^{\star}(x_t), p)) > 0\} \\
        &=  \indicator\{\ell(\Acal(x_t \mid L^{h^{\star}}_{B_{< t}}), \text{BinRel}(h^{\star}(x_t), p) )> 0\} \\
        &\leq \frac{1}{a} \, \ell(\Acal(x_t \mid L^{h^{\star}}_{B_{< t}}), \text{BinRel}(h^{\star}(x_t), p)),
    \end{split}
\end{equation*}
where the equality follows from the definition of the loss class. Here,  $a$ is the lower bound on $\ell$ whenever it is non-zero. Thus, we obtain
\[\mathbbm{E}\left[\sum_{t=1}^T \indicator\{E_{B, \phi_{B}^{h^{\star}}}(x_t) \neq h_i^{ \star,p}(x_t)\}  \right] \leq \frac{1}{a}\mathbbm{E}\left[\sum_{t=1}^T  \ell(\Acal(x_t \mid L^{h^{\star}}_{B_{< t}}), \text{BinRel}(h^{\star}(x_t), p)) \right]  \]

Now, we will again use the fact that the prediction $\Acal(x_t \mid L^{h^{\star}}_{B_{< t}})$ only depends on $(B_1, \ldots, B_{t-1})$, but is independent of $ B_t$. Using this independence, we obtain
\begin{equation*}
    \begin{split}
    \mathbbm{E}\left[\sum_{t=1}^T  \ell(\Acal(x_t \mid L^{h^{\star}}_{B_{< t}}), \text{BinRel}(h^{\star}(x_t), p)) \right] &=  \frac{T}{T^{\beta}}\, \mathbbm{E}\left[\sum_{t : B_t =1} \ell(\Acal(x_t \mid L^{h^{\star}}_{B_{< t}}), \text{BinRel}(h^{\star}(x_t), p)) \right] \\
    &= \frac{T}{T^{\beta}}\, \mathbbm{E} \left[ \mathbbm{E}\left[\sum_{t : B_t =1} \ell(\Acal(x_t \mid L^{h^{\star}}_{B_{< t}}), \text{BinRel}(h^{\star}(x_t), p)) \, \Bigg| B \right]  \right] \\
    &\leq \frac{T}{T^{\beta}} \expect\left[R(|B|, K) \right],
\end{split}
\end{equation*}
where $R(|B|, K)$ is the regret of the algorithm $\Acal$ and is a sub-linear function of $|B|$. In the last step, we use the fact that $\Acal$ is a (realizable) online learner for $\Hcal$ w.r.t. $\ell$ and the feedback that the algorithm received was $(x_t, \text{BinRel}(h^{\star}(x_t), p))$ in the rounds whenever $B_t =1$. Again, using Lemma 5.17 from \cite{woess2017groups} and Jensen's inequality yields $\mathbbm{E}_B\left[ R(|B|, K) \right] \leq  R(T^{\beta}, K)$, a sub-linear function of $T^{\beta}$. Combining everything, we get
\begin{equation*}
    \begin{split}
      \mathbbm{E}\left[\sum_{t=1}^T \indicator\{\mathcal{Q}(x_t) \neq h_i^{ \star,p}(x_t)\} \right] &\leq \sum_{t=1}^T \indicator\{h_i^{\star, p}(x_t) \neq y_t\}  +   \frac{T}{ a\, T^{\beta}} R(T^{\beta}, K) + \sqrt{2 T^{1+\beta} K \ln{ K}} \\
      &\leq \inf_{h_i^p \in \Hcal_i^p}\sum_{t=1}^T \indicator\{h_i^{p}(x_t) \neq y_t\}  + \frac{T}{a\, T^{\beta}} R(T^{\beta}, K) + \sqrt{2 T^{1+\beta} K \ln{ K}}
    \end{split}
\end{equation*}

For any choice of $\beta \in (0, 1)$, the regret above is a sub-linear function of $T$. Therefore, we have shown that $\Qcal$ is an agnostic learner for $\Hcal_i^p$ w.r.t. $0$-$1$ loss. This completes our proof.
\end{proof}

\section{Technical Lemmas}

Throughout this section, for any ranking (permutation) $\pi \in \mathcal{S}_K$, we let $\pi_i^j = \mathbbm{1}\{\pi_i \leq j\}$ for all $(i, j) \in [K]$.

\begin{lemma} \label{lem:subaddsum@p}
    For any $y \in \mathcal{Y}$,  $(\pi, \hat{\pi}) \in \mathcal{S}_k$, and $\ell \in \mathcal{L}(\ell_{\text{sum}}^{@p})$ 
    $$\ell(\pi, y) \leq \ell(\hat{\pi}, y) + c \:p\:\mathbbm{E}_{j \sim \text{Unif}([p])}\left[\ell(\pi, \text{BinRel}(\hat{\pi}, j)) \right].$$
    where $c = \frac{\max_{\tilde{\pi}, y}\ell(\tilde{\pi}, y)}{\min_{\tilde{\pi}, y}\{ \ell(\tilde{\pi}, y) \, \mid \, \ell(\tilde{\pi}, y) \neq 0\}}.$
\end{lemma}

\begin{proof}
Assume that $\ell(\pi, y) > \ell(\hat{\pi}, y) \geq 0$ (as otherwise the inequality trivially holds). Then, since $\ell \in \mathcal{L}(\ell_{\text{sum}}^{@p})$, it must be the case that $\hat{\pi} \stackrel{[p]}{\neq} \pi$. That is, $\hat{\pi}$ and $\pi$ assign different ranks to the labels in the top $p$. Therefore, there exists  $i \in [p]$ s.t. $\ell_{\text{sum}}^{@p}(\pi, \text{BinRel}(\hat{\pi}, i)) > 0$. Since $\ell \in \mathcal{L}(\ell_{\text{sum}}^{@p})$, for this same $i \in [p]$, $\ell(\pi, \text{BinRel}(\hat{\pi}, i)) > 0$. Therefore, we have 
\begin{align*}
c\:p\:\mathbbm{E}_{j \sim \text{Unif}([p])}\left[\ell(\pi, \text{BinRel}(\hat{\pi}, j)) \right] &\geq c\ell(\pi, \text{BinRel}(\hat{\pi}, i))\\
&= \frac{\max_{\tilde{\pi}, y}\ell(\tilde{\pi}, y)}{\min_{\tilde{\pi}, y}\{ \ell(\tilde{\pi}, y) \, \mid \, \ell(\tilde{\pi}, y) \neq 0\}}\ell(\pi, \text{BinRel}(\hat{\pi}, i))\\
&\geq \max_{\tilde{\pi}, y}\ell(\tilde{\pi}, y)\\
&\geq \ell(\pi, y).
\end{align*}
Combining the upperbounds in both cases gives the desired inequality. 
\end{proof}

\begin{lemma} \label{lem:subaddprec@p}
    For any $y \in \mathcal{Y}$,  $(\pi, \hat{\pi}) \in \mathcal{S}_k$, and $\ell \in \mathcal{L}(\ell_{\text{prec}}^{@p})$ 
    $$\ell(\pi, y) \leq \ell(\hat{\pi}, y) + c\:\ell(\pi, \text{BinRel}(\hat{\pi}, p)).$$
    where $c = \frac{\max_{\tilde{\pi}, y}\ell(\tilde{\pi}, y)}{\min_{\tilde{\pi}, y}\{ \ell(\tilde{\pi}, y) \, \mid \, \ell(\tilde{\pi}, y) \neq 0\}}.$
\end{lemma}

\begin{proof}
Assume that $\ell(\pi, y) > \ell(\hat{\pi}, y) \geq 0$ (as otherwise the inequality trivially holds). Then, since $\ell \in \mathcal{L}(\ell_{\text{prec}}^{@p})$, it must be the case that $\hat{\pi} \stackrel{p}{\neq} \pi$. That is, $\hat{\pi}$ and $\pi$ assign different labels in the top $p$. Therefore, $\ell_{\text{prec}}^{@p}(\pi, \text{BinRel}(\hat{\pi}, p)) > 0$. Since $\ell \in \mathcal{L}(\ell_{\text{prec}}^{@p})$, $\ell(\pi, \text{BinRel}(\hat{\pi}, p)) > 0$. Therefore, we have 
\begin{align*}
c\: \ell(\pi, \text{BinRel}(\hat{\pi}, p)) &= \frac{\max_{\tilde{\pi}, y}\ell(\tilde{\pi}, y)}{\min_{\tilde{\pi}, y}\{ \ell(\tilde{\pi}, y) \, \mid \, \ell(\tilde{\pi}, y) \neq 0\}}\ell(\pi, \text{BinRel}(\hat{\pi}, p))\\
&\geq \max_{\tilde{\pi}, y}\ell(\tilde{\pi}, y)\\
&\geq \ell(\pi, y).
\end{align*}
Combining the upperbounds in both cases gives the desired inequality. 
\end{proof}

\begin{lemma} \label{lem:sum@plb}
    Let $\pi, \hat{\pi} \in \mathcal{S}_k$. Then, for all $(i, j) \in [K] \times [p]$,  $\ell^{@p}_{\text{sum}}(\pi, \text{BinRel}(\hat{\pi}, j)) \geq \mathbbm{1}\{\pi_{i}^j \neq \hat{\pi}_{ i}^j\}$.
\end{lemma}

\begin{proof}
Fix label $i^{\star} \in [K]$ and threshold $j^{\star} \in [p]$. Our goal is to show that $\ell^{@p}_{\text{sum}}(\pi, \text{BinRel}(\hat{\pi}, j^{\star})) \geq \mathbbm{1}\{\pi_{i^{\star}}^{j^{\star}} \neq \hat{\pi}_{i^{\star}}^{j^{\star}}\}$.  Recall that $\text{BinRel}(\hat{\pi}, j^{\star})[i^{\star}] = \mathbbm{1}\{\hat{\pi}_{i^{\star}} \leq j^{\star}\}$ by definition. Since $\ell_{\text{sum}}^{@p}(\hat{\pi}, \text{BinRel}(\hat{\pi}, j^{\star})) = 0$, we have that 

\begin{align*}
    \ell_{\text{sum}}^{@p}(\pi, \text{BinRel}(\hat{\pi}, j^{\star})) &= \ell^{@p}_{\text{sum}}(\pi, \text{BinRel}(\hat{\pi}, j^{\star})) - \ell^{@p}_{\text{sum}}(\hat{\pi}, \text{BinRel}(\hat{\pi}, j^{\star}))\\
    &= \sum_{i=1}^K \min(\pi_{i}, p+1) \text{BinRel}(\hat{\pi}, j^{\star})[i] - \sum_{i=1}^K \min(\hat{\pi}_{i}, p+1) \text{BinRel}(\hat{\pi}, j^{\star})[i]\\
    &= \sum_{i=1}^K \min(\pi_{i}, p+1) \mathbbm{1}\{\hat{\pi}_{i} \leq j^{\star}\}  - \sum_{i=1}^K \min(\hat{\pi}_{i}, p+1) \mathbbm{1}\{\hat{\pi}_{i} \leq j^{\star}\}\\
    &= \sum_{i=1}^K \min(\pi_{i}, p+1) \mathbbm{1}\{\hat{\pi}_{i} \leq j^{\star}\}  - \sum_{i=1}^K \hat{\pi}_{i} \mathbbm{1}\{\hat{\pi}_{i} \leq j^{\star}\}
\end{align*}

Let $\mathcal{I} \subset [K]$ s.t. for all $i \in \mathcal{I}$, $\hat{\pi}_{i}^{j^{\star}} = \mathbbm{1}\{\hat{\pi}_{i} \leq j^{\star}\} = 1$. Then, we have that 

\begin{align*}
\ell_{\text{sum}}^{@p}(\pi, \text{BinRel}(\hat{\pi}, j^{\star})) &= \sum_{i \in \mathcal{I}}\min(\pi_{i}, p+1) - \sum_{i \in \mathcal{I}}\hat{\pi}_{i}\\
&= \sum_{i \in \mathcal{I}} \min(\pi_{i}, p+1) - \sum_{i=1}^{j^{\star}} i
\end{align*}

Suppose that $\mathbbm{1}\{\pi_{i^{\star}}^{j^{\star}} \neq \hat{\pi}_{i^{\star}}^{j^{\star}}\} = 1$. It suffices to show that $\ell_{\text{sum}}^{@p}(\pi, \text{BinRel}(\hat{\pi}, j^{\star})) \geq 1$. There are two cases to consider. Suppose $i^{\star} \in \mathcal{I}$. Then, it must be the case that $\mathbbm{1}\{\pi_{i^{\star}} \leq j^{\star}\} = \pi_{i^{\star}}^{j^{\star}} = 0$, implying that $\pi_{i^{\star}} \geq j^{\star} + 1$.  It then follows that in the best case $\sum_{i \in \mathcal{I}} \min(\pi_{i}, p+1) \geq  \sum_{i=1}^{j^{\star}-1} i + (j^{\star} + 1) > \sum_{i=1}^{j^{\star}} i$ showcasing that indeed $\ell_{\text{sum}}^{@p}(\pi, \text{BinRel}(\hat{\pi}, j)) \geq 1$. Now, suppose $i^{\star} \notin \mathcal{I}$. Then, $\mathbbm{1}\{\hat{\pi}_{i^{\star}} \leq j^{\star}\} = 0$, which means that $\mathbbm{1}\{\pi_{i^{\star}} \leq j^{\star}\} = 1$. Accordingly, while $\hat{\pi}$ did not rank label $i^{\star}$ in the top $j^{\star}$, $\pi$ \textit{did} rank label $i^{\star}$ in the top $j^{\star}$. Since $|\mathcal{I}| = j^{\star}$, there must exist an label $\hat{i} \in \mathcal{I}$ which $\pi$ does not rank in the top $j^{\star}$. That is, there exists $\hat{i} \in \mathcal{I}$ s.t. $\pi_{\hat{i}} \geq j^{\star} + 1$.  Using the same logic,  in the best case $\sum_{i \in \mathcal{I}} \min(\pi_{i}, p+1) \geq  \sum_{i=1}^{j-1} i + (j^{\star} + 1)$ showcasing that again $\ell_{\text{sum}}^{@p}(\pi, \text{BinRel}(\hat{\pi}, j^{\star})) \geq 1$. Thus, we have shown that when $\mathbbm{1}\{\pi_{i^{\star}}^{j^{\star}} \neq \hat{\pi}_{i^{\star}}^{j^{\star}}\} = 1$, $\ell_{\text{sum}}^{@p}(\pi, \text{BinRel}(\hat{\pi}, j^{\star})) \geq 1$. Since $i^{\star}$ and $j^{\star}$ were arbitrary, this must be true for any $(i, j) \in [K] \times [p]$, completing the proof. \end{proof}

\begin{lemma} \label{lem:prec@plb}
    Let $\pi, \hat{\pi} \in \mathcal{S}_k$. Then, for all $i \in [K]$,  $\ell^{@p}_{\text{prec}}(\pi, \text{BinRel}(\hat{\pi}, p)) \geq \mathbbm{1}\{\pi_{i}^p \neq \hat{\pi}_{i}^p\}$.
\end{lemma}

\begin{proof}
Fix label $i^{\star} \in [K]$. Our goal is to show that $\ell^{@p}_{\text{prec}}(\pi, \text{BinRel}(\hat{\pi}, p)) \geq \mathbbm{1}\{\pi_{i^{\star}}^{p} \neq \hat{\pi}_{i^{\star}}^{p}\}$.  Recall that $\text{BinRel}(\hat{\pi}, p)[i^{\star}] = \mathbbm{1}\{\hat{\pi}_{i^{\star}} \leq p\}$ by definition. Since $\ell_{\text{prec}}^{@p}(\hat{\pi}, \text{BinRel}(\hat{\pi}, p)) = 0$, we have that 
\begin{align*}
    \ell_{\text{prec}}^{@p}(\pi, \text{BinRel}(\hat{\pi}, p)) &= \ell^{@p}_{\text{prec}}(\pi, \text{BinRel}(\hat{\pi}, p)) - \ell^{@p}_{\text{prec}}(\hat{\pi}, \text{BinRel}(\hat{\pi}, p))\\
    &= \sum_{i=1}^K \mathbbm{1}\{\hat{\pi}_{i} \leq p\} \text{BinRel}(\hat{\pi}, p)[i] - \sum_{i=1}^K \mathbbm{1}\{\pi_{i} \leq p\} \text{BinRel}(\hat{\pi}, p)[i] \\
    &= p - \sum_{i=1}^K  \mathbbm{1}\{\pi_{i} \leq p\} \mathbbm{1}\{\hat{\pi}_{i} \leq p\}  \\
\end{align*}
Let $\mathcal{I} \subset [K]$ s.t. for all $i \in \mathcal{I}$, $\hat{\pi}_{i}^{p} = \mathbbm{1}\{\hat{\pi}_{i} \leq p\} = 1$. Then, we have that 

$$\ell_{\text{prec}}^{@p}(\pi, \text{BinRel}(\hat{\pi}, p)) = p - \sum_{i \in \mathcal{I}}\mathbbm{1}\{\pi_{i} \leq p\}.$$ 

Suppose that $\mathbbm{1}\{\pi_{i^{\star}}^{p} \neq \hat{\pi}_{i^{\star}}^{p}\} = 1$. It suffices to show that $\ell_{\text{prec}}^{@p}(\pi, \text{BinRel}(\hat{\pi}, p)) \geq 1$. There are two cases to consider. Suppose $i^{\star} \in \mathcal{I}$. Then, it must be the case that $\mathbbm{1}\{\pi_{i^{\star}} \leq p\} = \pi_{i^{\star}}^{p} = 0$, implying that $\pi_{i^{\star}} \geq p + 1$.  It then follows that in the best case $\sum_{i \in \mathcal{I}} \mathbbm{1}\{\pi_{i} \leq p\} \leq p - 1 < p$ showcasing that indeed $\ell_{\text{sum}}^{@p}(\pi, \text{BinRel}(\hat{\pi}, p)) \geq 1$. Now, suppose $i^{\star} \notin \mathcal{I}$. Then, $\mathbbm{1}\{\hat{\pi}_{i^{\star}} \leq p\} = 0$, which means that $\mathbbm{1}\{\pi_{i^{\star}} \leq p\} = 1$. Accordingly, while $\hat{\pi}$ did not rank label $i^{\star}$ in the top $p$, $\pi$ \textit{did} rank label $i^{\star}$ in the top $p$. Since $|\mathcal{I}| = p$, there must exist an label $\hat{i} \in \mathcal{I}$ which $\pi$ does not rank in the top $p$. That is, there exists $\hat{i} \in \mathcal{I}$ s.t. $\pi_{\hat{i}} \geq p + 1$.  Using the same logic,  in the best case $\sum_{i \in \mathcal{I}} \mathbbm{1}\{\pi_{i} \leq p\} \leq p-1 < p $ showcasing that again $\ell_{\text{prec}}^{@p}(\pi, \text{BinRel}(\hat{\pi}, p)) \geq 1$. Thus, we have shown that when $\mathbbm{1}\{\pi_{i^{\star}}^{p} \neq \hat{\pi}_{i^{\star}}^{p}\} = 1$, $\ell_{\text{prec}}^{@p}(\pi, \text{BinRel}(\hat{\pi}, p)) \geq 1$. Since $i^{\star}$ was arbitrary, this must be true for any $i \in [K]$, completing the proof. \end{proof}

\end{document}